\documentclass[accepted]{uai2026} 
                        
\usepackage[american]{babel}

\usepackage{natbib} 
\usepackage{mathtools} 
\usepackage{booktabs} 
\usepackage{tikz} 
\usepackage{amsthm, amsfonts, amsmath, amssymb}

\newtheorem{theorem}{Theorem}

\newtheorem{lemma}[theorem]{Lemma}

\usepackage{algorithm}
\usepackage{algpseudocode}
\usepackage{graphicx}

\usepackage{xurl}

\usepackage{subcaption}
\usepackage{enumitem}
\usepackage{caption}
\usepackage{hyperref}
\usepackage{colortbl}
\usepackage{ragged2e} 
\usepackage{tabularx}
\usepackage{makecell}
\usepackage[normalem]{ulem}
\usepackage{array}
\usepackage{graphicx}
\usepackage{caption}
\usepackage{booktabs, multirow, xcolor}
\usepackage{siunitx}       
\usepackage{pgfplots}
\pgfplotsset{compat=1.17}
\usepackage{tikz}
\usetikzlibrary{patterns}
\usepackage[table]{xcolor}
\usepackage{microtype}
\usetikzlibrary{patterns}
\usepackage{pgfplots}
\pgfplotsset{compat=1.17}
\usepackage{tikz}
\usetikzlibrary{patterns}
\usepackage{tikz}
\usepackage{pgfplots}
\usepackage{subcaption}
\usepackage{caption}
\usetikzlibrary{matrix, shapes.geometric}

\pgfplotsset{compat=1.17}
\usepgfplotslibrary{fillbetween}

\definecolor{colUS}{RGB}{31,119,180}   
\definecolor{colETC}{RGB}{44,160,44}   
\definecolor{colSR}{RGB}{255,127,14}   
\definecolor{colFCSR}{RGB}{214,39,40}  

\pgfplotsset{
  every axis/.append style={
    tick label style={font=\small},
    label style={font=\small},
    title style={font=\small},
    legend style={font=\small},
    scaled ticks=false
  }
}

\newtheorem{remark}{Remark}

\title{Fixed-Budget Constrained Best Arm Identification in Grouped Bandits}

\author[1]{\href{mailto:<raunakm@iitb.ac.in>?Subject=Your UAI 2026 paper}{Raunak Mukherjee}{}}
\author[2]{\href{mailto:<sharayum@iitb.ac.in>?Subject=Your UAI 2026 paper}{Sharayu Moharir}{}}

\affil[1]{%
    Dept. of Electrical Engineering\\
    IIT Bombay\\
}
\affil[2]{%
    Dept. of Electrical Engineering\\
    IIT Bombay
}
  
\begin{document}
\maketitle

\begin{abstract}
   We study fixed budget constrained best-arm identification in grouped bandits, where each arm consists of multiple independent attributes with stochastic rewards. An arm is considered feasible only if all its attributes' means are above a given threshold. The aim is to find the feasible arm with the largest overall mean. We first derive a lower bound on the error probability for any algorithm on this setting. We then propose Feasibility Constrained Successive Rejects (FCSR), a novel algorithm that identifies the best arm while ensuring feasibility.  We show it attains optimal dependence on problem parameters up to constant factors in the exponent. Empirically, FCSR outperforms natural baselines while preserving feasibility guarantees. 
\end{abstract}

\section{Introduction}\label{sec:intro}

The aim of the pure exploration stochastic multi-armed
bandit problem is to identify the optimal arm among a given set of arms. The two most popular paradigms for best arm identification (BAI) include the \emph{fixed confidence} setting \citep{kaufmann2016complexitybestarmidentification} and the \emph{fixed budget} setting \citep{Audibert2010BestArmIdentification}.  \emph{Fixed budget best arm identification} (FBBAI) has numerous applications in online advertising, recommender systems, etc. In general it applies to situations that contain a finite \emph{testing} phase and then a \emph{commercialization} phase \citep{Audibert2010BestArmIdentification}. 

Many services are more naturally modeled as an aggregation of independent services that are rendered to a customer. It is often desirable that each of these services are above a certain standard of quality \citep{dharod2024constrainedbestarmidentification}. For instance, a typical auto garage may offer car wash services, AC servicing, tyre and wheel care services, car inspections, etc. To evaluate such services, it makes sense to have customers rate each service separately and maintain a
rating for each service. A reasonable metric for evaluating
such a service as a whole is the (weighted) average of
the ratings of the different services. In addition, for an
amenity to be deemed acceptable, it may be desirable that
the ratings for each service exceed a threshold


Similarly, in online advertising, one might seek the best creative (such as an image or video ad) that is run across multiple demographic segments. Each creative may be treated as an arm, with the (random) performance over each demographic segment modeled as an attribute of that creative. We desire the best creative that does not perform unacceptably poorly on any demographic segment.

We model this as a pure exploration MAB problem in a grouped bandit setting, where each arm is a group of attributes that are modeled as random variables that are sampled independently. An arm is
said to be feasible if the mean reward of all its attributes
exceeds a given threshold.

\paragraph{Our Contributions.}  We propose Feasibility Constrained Successive Rejects (FCSR), a novel hybrid sampling algorithm that attains optimal dependence on problem parameters up to a constant in the exponent. FCSR is \emph{entirely parameter free} \footnote{See \cite{DBLP:journals/corr/abs-2105-13017} for similar usage of the terms "parameter free" and "optimal" to describe algorithms.} in the sense that it requires no knowledge of the problem instance (which is usually unavailable in practice). FCSR is a hybrid sampling strategy that synthesizes existing fixed-budget algorithms in addition to incorporating a novel sampling heuristic ($\textsc{SampleUntilFeasible}$).

We demonstrate this by first defining a new complexity parameter for this setting $H_{FC}$, that recovers the well-known hardness parameter for the fixed budget best arm identification setting if the constraint is relaxed. We derive a non-asymptotic lower bound on the error probability in terms of the instance dependent parameter $H_{FC}$. We then prove a matching upper bound on the probability of error of FCSR, thereby establishing \emph{optimality up to a constant factor in the exponent.} We test FCSR against baselines considered in similar settings on several synthetic as well as real-world datasets. Our results show that FCSR outperforms baselines on the whole. 

\paragraph{Related Work.} 

The unconstrained single attribute pure exploration MAB formulation is well studied in the fixed budget setting. Algorithms such as Successive Rejects \citep{Audibert2010BestArmIdentification} and Sequential Halving \citep{pmlr-v28-karnin13} divide the exploration budget into phases, and eliminate/reject one or more arms at the end of each phase using an elimination rule. The last surviving arm is then flagged as optimal. They are among the best known algorithms in this setting while being relatively simple.  We note that the in the unconstrained case, the decision on which arm to reject is relatively straightforward (rank each arm by its empirical mean). In contrast, \citet{Faizal2022ConstrainedPE} note that in the constrained setting the pursuit of the arm with the largest empirical mean must be balanced with the potential of violating the constraint .



\citet{wang2021bestarmidentificationsafety} consider linear and then monotonic safety constraints on a best arm identification problem. \citet{hou2022optimalvarianceconstrainedbestarm} consider a \emph{fixed confidence} best arm identification problem with a feasibility constraint is placed on the variance of the arm. These works differ from ours since they do not consider identifying arms that are comprised of multiple attributes. \citet{pmlr-v89-katz-samuels19a} considers the problem of identifying the best multidimensional arm that satisfies a general feasibility constraint, however, the work focuses on the fixed confidence setting where feedback is multidimensional, in contrast to the fixed budget single-dimensional feedback setting in this paper.

\citet{locatelli2016optimalalgorithmthresholdingbandit} introduced the thresholding bandit problem (TBP) and an optimal algorithm APT for the setting. In TBP, there are $K$ scalar-valued distributions and a threshold $\tau$. The goal is to identify all distributions with means greater than $\tau$ within a fixed budget $T$. \citet{Katz2018} generalize the TBP setting by considering multidimensional arms and the problem of identifying those arms with means belonging to a given polyhedron as mentioned above. The constraint considered in the above works is very similar to the feasibility constraint considered here. However, the above papers are concerned with identifying \emph{all} arms that satisfy the above constraints as opposed to finding the \emph{best} arm that fulfills them.

\citet{kagrecha2023constrainedregretminimizationmulticriterion} considers a similar multi-dimensional constrained bandit setting but with a \emph{regret minimization} objective. The best arm identification objective is considered in \citet{dharod2024constrainedbestarmidentification} and in \citet{Faizal2022ConstrainedPE}. However, the former considers the case where the aim is to maximize a singular attribute subject to a constraint on the other, and the latter considers the same MAB setup as in this paper, but is set in the fixed confidence regime. We provide results for the fixed-budget setting which, to the best of our knowledge, is an open problem.


\section{PROBLEM FORMULATION}
\paragraph{Constrained Grouped Bandit Setting.} We have $K$ arms each with $M$ attributes. Let $[K]=\{1,2,\dots,K\}$ and $[M] = \{1, 2, \dots, M\}$. 
Attribute $j$ of arm $i$ is denoted as the tuple $(i, j)$, and is associated with the unknown reward distribution $\nu_{i,j}$. The random rewards of all attributes are i.i.d and assumed to be $R$ sub-gaussian. 
We define a \emph{bandit instance} defined as the product distribution $\mathcal{B} = \bigotimes_{(i, j) \in \mathcal{A}}\nu_{ij}$.

\textbf{Definition ($R$-sub-Gaussian distribution).} Let $R > 0$. A distribution $\nu$ is $R$-sub-Gaussian if for all $t \in \mathbb{R}$ we have $\mathbb{E}_{X \sim \nu}[\exp(t(X - \mathbb{E}[X])] \leq \exp\left(\frac{R^2 t^2}{2}\right)$.

The \textit{mean reward} of attribute $(i,j)$ is defined as $\mu_{i,j} := \mathbb{E}_{X \sim \nu_{i,j}}[X]$. The \textit{mean reward} of arm $i$ is the average of its attribute means and is defined as 
\begin{equation}
\mu_i := \frac{1}{M}\sum_{j=1}^M \mu_{i,j}.
\label{eq:arm-mean}
\end{equation}

In this work, we consider the simple average of all attributes, however, results may easily be extended to the weighted average case. An arm $i$ is called \emph{feasible} iff the mean reward of all attributes $(i,j)$ is above a given threshold $\tau \in\mathbb{R}$. We define the set of all feasible arms as $\mathcal{F}:=\{i\in[K]:\min_{j\in[M]}\mu_{i,j} > \tau\}$. Let $\mathcal{F}^c := [K] \backslash \mathcal{F}$ be the set of all \emph{infeasible} arms. 


The best feasible arm, if it exists, is 
\begin{equation*}
i^\star := \arg\max_{i\in\mathcal{F}} \mu_i.
\label{eq:best-arm}
\end{equation*}
We assume the existence of a unique best arm. The generalization to multiple optimal arms is straightforward. If there exists no feasible arm, we declare the given instance \emph{infeasible}. Let us denote this case by defining $i^\star=0$ as the \emph{flag} denoting infeasibility of the instance.

\paragraph{Objective.} Given an unknown bandit instance $\mathcal{B}$, at any time $t$, the learner chooses an attribute $(i_t, j_t)$ and receives
a random reward drawn from the distribution $X_{i,j}(t) \sim \nu_{i,j}$. An adaptive learner bases its decision at
time $t$ on the samples observed in the past and outputs a decision $I_t \in \{0, 1, \dots, K\}$ after it has exhausted its budget $T$. The goal of the learner is to minimize the probability it deviates from the true $i^\star$. Formally, given a budget $T$ and a threshold $\tau$ known to the learner, and a bandit instance $\mathcal{B}$ unknown to the learner, the learner aims to \emph{minimize} 
\begin{equation}
\mathbb{P}_{\mathcal{B}}( I_t \neq i^\star).
\end{equation}

\section{LOWER BOUND}
In this section, we provide a fundamental lower bound on the minimum probability of error incurred by any policy in this setting.

Let $\mathcal{R} = \{i: i \in  \mathcal{F}^c, \mu_i \geq \mu_{i^\star}\}$. Note that if $i^\star = 0$, then $\mathcal{R} = \mathcal{F}^c$. Henceforth, referred to as \emph{risky} arms, $\mathcal{R}$ contains all infeasible arms with higher average mean reward than the optimal arm.
 There are three distinct possibilities in the case that the adaptive learner incorrectly report the best arm, i.e. $I_T \neq i^\star$: (i) the best arm $i^\star$ exists, but is incorrectly deemed infeasible, (ii) a \emph{feasible sub-optimal arm} $i' \in \mathcal{F}_S$ is deemed to possess a greater mean reward than $i^\star$ or (iii) a \emph{risky} arm $r \in \mathcal{R}$ is deemed to be feasible (and by definition $r$ beats $i^\star$ in mean reward). The case when no best arm exists and $i^\star = 0$ is a subset of scenario (iii).

\paragraph{Complexity Parameter Definitions.} We define some parameters related to the hardness of the problem instance in this setting. Define the \emph{threshold gap} for attribute $(i,j)$ as $\bar{\Delta}_{i,j} := |\mu_{i,j} - \tau|$. Further, set $\bar{\Delta}_{0,j} = \infty \,\forall\, j \in [M]$. Define the \emph{sub-optimality} gap for arm $i\neq i^\star$ as $ \Delta_i := |\mu_{i^\star} - \mu_i| \,\forall \, i \in \{0, \dots, K\}$. Set $\mu_0 := \infty$ and thus if the best arm does not exist ($i^\star = 0$), $\Delta_i = \infty \,\forall \,i$. The definitions of $\bar{\Delta}_{0,j}$ and $\mu_0$ are made such that the parameters below remain well defined even if the best arm does not exist. 

Assume all arms are indexed by their average mean reward, i.e. $\mu_1 \geq \mu_2 \geq \dots \geq \mu_K$. We define the parameters \emph{mean hardness}, \emph{risky hardness} and \emph{feasibility hardness}, respectively on some bandit problem $\mathcal{B}$ as  
\begin{align}
\label{eq:hardness-param}
    & H_{2}^R (\mathcal{B}) := \max_{i \in [|\mathcal{R}|+2:K]} i\Delta_{i}^{-2},
    \\ &H_{f} (\mathcal{B}) :=\frac{K}{\log(K)} \max_{j \in [M]}
    \bar{\Delta}_{i^\star j}^{-2} \,, 
    \\ & H_{\text{tbp}}(i) := \sum_{j \in [M]}\bar{\Delta}_{i,j}^{-2} \, \,
    , \, \, H_{\text{tbp}} (\mathcal{B}) := K\max_{i\in \mathcal{F}^C}H_{\text{tbp}}(i).
\end{align}
Let $H_{\text{tbp}} := 0$ if $|\mathcal{F}^C| = 0$ i.e. there are no infeasible arms. The dependence on $\mathcal{B}$ is often implicit. Hence, define the overall \emph{feasibility constrained} hardness parameter of a bandit instance $\mathcal{B}$ as 
\begin{equation}
    H_{FC}(\mathcal{B}) := \max \left\{ H_2^R, \, H_{\text{tbp}}, \,  H_f\right\}.
\end{equation}

\begin{remark}
    Note that in the case we allow $\tau = -\infty$, all arms are feasible a.s. and the problem instance is identical to unconstrained best arm identification with grouped arms. We have $H_f = 0$ and $H_{tbp} = 0$. Thus, $H_{FC} = H_2 = \max_{i \in [K]}i\Delta_i^{-2}$ and $H_{FC}$ recovers the hardness index widely seen in the fixed budget unconstrained setting (see \citet{Audibert2010BestArmIdentification}). This is the hardness index we expect to see in vanilla best arm identification if we treat each grouped arm as a single dimensional arm with the same arm mean. 
\end{remark}


\begin{theorem}[Lower Bound]
 Let $\mathcal{C}_{FC}(a;K)$ denotes the set of bandit instances with $K=M \geq 2$ and whose difficulty $H_{FC}$ is upper bounded by some constant $a$. If
 \begin{equation}
 T \geq \frac{4}{60^2}(a\log(K))^2\log(6T(K+M+1)),
 \end{equation}
 there exists a bandit instance $\mathcal{G} \in \mathcal{C}_{FC}(a)$ such that the probability any arbitrary learner incorrectly reports the best arm is at least
 \begin{equation}
 \label{eq:lb-p}
\mathbb{P}_{\mathcal{G} \in \mathcal{C}_{FC}(a)}\geq \frac{1}{6} \exp\left( \frac{-1200T}{\log(K)H_{FC}(\mathcal{G})} \right).
\end{equation}
\end{theorem}


\textit{Proof Sketch and Discussion.} The proof proceeds by constructing two adversarial families of bandit problems that are subsets of $\mathcal{C}_{FC}(a)$ and capture two principal failure: (i) a \emph{feasibility} family $\mathcal{C}_F$ in which only one arm crosses the threshold $\tau$ and (ii) a \emph{risky} family $\mathcal{C}_R$ that elevates competitor arms or demotes single attributes of the best arm. We adapt the technique used by \citet{carpentier2016tightlowerboundsfixed} to this more general setting. For each family, a concentration event $\xi$ is established and a change-of-measure argument is applied to some undersampled arm giving us the desired result. Taking the maximum over the two families gives the stated bound with the combined hardness $H_{FC}=\max\{H_{tbp},H_2^R,H_f\}$. The main new contribution is the construction of a rich class of multi-dimensional bandit instances that capture multiple failure modes between violating the feasibility constraint and error in mean discrimination. See Appendix \ref{app:lb} for the full proof.

\section{FEASIBILITY CONSTRAINED SUCCESSIVE REJECTS}
\subsection{Notation}
Define $[i:j] = \{i, i + 1, \dots, j\}$ and $[j] = [1:j]$ where $i$ and $j$ are integers. Let $X_{i,j}(t) \sim \nu_{i,j}$ denote the random reward observed on pulling attribute $(i,j)$ at time $t > 0$ (Unless otherwise mentioned $i \in [K], j\in [M]$). Let $C_{i,j}(t)$ denote the number of times $(i,j)$ was pulled before time $t$, and $S_{i,j}$ the sum of all observed rewards before $t$. Also define the attribute empirical mean $\hat{\mu}_{i,j}(t) := S_{i,j} (t) / \max \{C_{i,j} (t), 1\}$. Thus, on pulling attribute $(i,j)$ at time $t$ we make the following update 
\begin{equation}
\label{eq:upd-rule-p}
\begin{aligned}
& S_{i,j}(t+1) \gets S_{i,j}(t) + X_{i,j}(t)\,, \\
& C_{i,j}(t+1) \gets C_{i,j}(t) + 1\,, \\
&\hat{\mu}_{i,j}(t+1) = \frac{S_{i,j} (t+1)}{\max \{C_{i,j} (t+1), 1\}}\,,
\end{aligned}
\end{equation}
and $S_{k,m}(t+1) = S_{k,m}(t) \;,\; C_{k,m}(t+1) = C_{k,m}(t) \;,\; \hat{\mu}_{k,m}(t+1) = \hat{\mu}_{k,m}(t)$ remain unchanged for all other attributes $(k,m) \neq (i,j)$.

Let $S(t) \in \mathbb{R}^{K \times M}$ be a matrix where each entry is given by $S_{i,j}(t)$ and the vector of attribute sums of arm $i$, $S_{i}(t)$ be row $i$ of $S(t)$. Similarly, define $C(t) \in \mathbb{N}^{K \times M}$, $C_{i}(t) \in \mathbb{N}^{M}$ and the matrix of attribute empirical means $\hat{\mu}(t) \in \mathbb{R}^{K \times M}$. 

The FCSR algorithm maintains global attribute statistics matrices $S, C, \hat{\mu}$ that are updated by inner sampling sub-routines. Thus, for the sake of simplicity, we omit the time step $t$ when referring to the global attribute statistics $S, C, \hat{\mu}$ and the reward $X_{i,j}$. However, we explicitly index any statistic that is used only within the scope of a sub-routine. 

\subsection{Algorithm description}

We briefly recall the two fixed budget bandit algorithms that FCSR builds upon. \emph{Successive Rejects} (SR) \citep{Audibert2010BestArmIdentification} is a fixed-budget best arm identification strategy that proceeds in $K-1$ rounds. In each round, all surviving arms are sampled uniformly according to a prescribed schedule and an empirical estimate of each arm’s mean is computed. The arm with the worst estimate is eliminated. 

The \emph{APT} procedure for the Thresholding Bandit Problem \citep{locatelli2016optimalalgorithmthresholdingbandit} addresses a different problem: given arms with unknown means and a known threshold, it allocates samples adaptively to decide whether each mean lies above or below the threshold, concentrating samples on arms whose empirical means are close to the threshold. APT achieves optimal performance Algorithm \ref{alg:apt-p} describes the APT sampling procedure for the set of attributes of arm $i$ with an arbitrary initialization of the global arm statistic vectors $S_i, C_i$. Both SR and APT achieve optimal performance upto constant factors in the exponent \citep{carpentier2016tightlowerboundsfixed, locatelli2016optimalalgorithmthresholdingbandit}.

\begin{algorithm}[H]
\caption{$\textsc{APT}(i, T; S_i, C_i)$}
\label{alg:apt-p}
\begin{algorithmic}
\Require Arm $i$, Budget $T$, Threshold \(\tau \in \mathbb{R}\), and global statistics $S_i, C_i$
\State \textbf{Initialization:} Use existing global $S_i, C_i, \hat{\mu}_{i,j} \;\forall\; j \in [M]$. 
\For{\(t = 0, 1, \dots, T-1\)}
    \For{each attribute \(j = 1, \dots, M\)}
        \State Compute the empirical gap:
        $\widehat{\bar{\Delta}}_{j}(t) = |\hat{\mu}_{i,j} - \tau| $
        \State Compute: $B_j(t+1) = \sqrt{C_{i,j}} \cdot \widehat{\bar{\Delta}}_{ij}(t)$
    \EndFor
    \State Pull attribute \(J_{t+1} = \arg\min_{1 \leq j \leq M} B_j(t+1)\).
    \State Observe reward \(X_{i,J_{t+1}} \sim \nu_{i, J_{t+1}}\).
    \State Update global $S_{i, J_{t+1}}, C_{i,J_{t+1}}, \hat{\mu}_{i, J_{t+1}}$ using  \eqref{eq:upd-rule-p}.
\EndFor
\State \Return Updated statistics $S_i, C_i$.
\end{algorithmic}
\end{algorithm}

\begin{algorithm}
\caption{$\textsc{Uniform}(i, T; S_i, C_i)$}
\label{alg:us-p}
\begin{algorithmic}
\Require Arm $i$, Budget $T$ and global statistics $S_i, C_i$
\For{attribute $j = 1, 2, \dots M$ of arm $i$}
    \For{samples $s = 1, 2, \dots \lfloor T/M\rfloor$} 
        \State Observe reward $X_{i,j} \sim \nu_{i,j}$
        \State Update $S_{i, J_{t+1}} \,, \, C_{i,J_{t+1}}$, $\hat{\mu}_{i, J_{t+1}}$ using  \eqref{eq:upd-rule-p}.
    \EndFor
\EndFor
\State \Return Updated statistics $C_i, S_i$
\end{algorithmic}
\end{algorithm}

FCSR combines these two ideas: SR provides the global elimination schedule across arms based on estimated arm means, while APT is used locally to focus samples on attributes that are close to the feasibility threshold. While SR eliminates sub-optimal arms, APT sampling is targeted at eliminating risky arms. In addition, a novel sampling strategy, $\textsc{SampleUntilFeasible}$ (SUF) allocates a dedicated portion of the overall budget toward ensuring the best feasible arm is not erroneously eliminated. SUF selectively samples infeasible attributes to prevent the best arm from being eliminated early. This combination naturally leads to a three-phase structure described in more detail below.

Fix a total sample budget \(T\) and fractional hyperparameters \(f,g\in(0,1)\). We reserve a total fraction of the budget \(fT\) that is split equally: for every arm \(i\) set the feasibility budget $P_i$
\[
P_i \leftarrow \big\lfloor fT/K\big\rfloor.
\]
Initialize the shared leftover feasibility pool \(T_{\text{extra}}\leftarrow 0\). When arm $i$ is eliminated and it's feasibility budget $P_i > 0$, we repurpose those unused samples by adding it to the "extra" sample pool $T_{\text{extra}} \gets T_{\text{extra}} + P_i$ which is then used to further sample attributes uniformly (see the Uniform phase below). The remaining \((1-f)T\) samples are allocated across rounds following the SR schedule.

In the vanilla SR algorithm, $n_r$ is the total number of samples received by all arms at the end of round $r$. Thus, $\Delta n_r := n_r-n_{r-1}$ is the "per-round" budget for each surviving arm. Formally, 
\begin{equation}
\label{eq:dnr-defp}
n_r := \Big\lceil \frac{(1-f)T}{\bar n (K+1-r)}\Big\rceil \,, \quad \Delta n_r := n_r-n_{r-1}
\end{equation}
with \(n_0:=0\). $\bar{n} = \frac{1}{2} + \sum_{k=2}^{K} \frac{1}{k}$ is a normalizing constant. In round \(r=1,\dots,K-1\) let \(S_r\) denote the set of surviving arms. Each surviving arm \(i\in S_r\) receives three sequential sampling phases (all integer quantities are rounded as described below). Let \(|S_r|\) denote the number of surviving arms and \(\mathcal{A}=[K]\) be the set of all arms.
\begin{enumerate}[]
    \item \textbf{Uniform phase.} Allocate \(\big\lfloor(1-g)\Delta n_r\big\rfloor\) samples to each arm \(i\) uniformly across that arm's attributes (Algorithm \ref{alg:us-p}). In addition, redistribute \(\big\lfloor T_{\text{extra}}/|S_r|\big\rfloor\) samples from the leftover pool \(T_{\text{extra}}\) to each surviving arm; reduce \(T_{\text{extra}}\) accordingly.    
    \item \textbf{Risky (APT) phase.} Allocate \(\big\lfloor g\Delta n_r\rfloor\) samples \emph{per arm} according to the APT thresholding subroutine (Algorithm~\ref{alg:apt-p}). The APT allocation for arm \(i\) is applied after the uniform-phase observations of round \(r\).    
    \item \textbf{Feasibility phase.} For any attribute \(j\) of arm \(i\) that remains \emph{empirically infeasible} draw additional samples on that attribute sequentially, up to the arm's remaining feasibility budget \(P_i\). We stop sampling attribute \((i,j)\) when it is empirically feasible.  Deduct each feasibility sample from \(P_i\). If arm \(i\) is discarded later its remaining \(P_i\) is transferred to the pool: \(T_{extra}\leftarrow T_{\text{extra}}+P_i\). See \textsc{SampleUntilFeasible} (Algorithm \ref{alg:suf-p}).
\end{enumerate}

After all phases in round \(r\) compute an elimination score for each surviving arm \(i\). The score for arm $i$ at some time $t$ is defined as
\begin{equation}
    \label{eq:score-supm-p}
    s(i;t) :=
    \begin{cases}
        \hat\mu_i(t) &\text{if }\min_{j\in[M]}\hat\mu_{ij}(t)>\tau,\\
        \min_{j\in[M]}\hat\mu_{ij}(t)  &\text{otherwise.}
    \end{cases}
\end{equation}
 Repeat until one arm remains; return that arm if it is feasible, else return $-1$. For the complete pseudocode, refer Algorithm \ref{alg:fcsr-p}.

\begin{algorithm}[H]
\caption{\textsc{SampleUntilFeasible}$(i, P_i; S_{i}, C_{i}, \tau)$}
\label{alg:suf-p}
\begin{algorithmic}
\Require Arm $i$, feasibility budget $P_i$, global $S_i \in \mathbb{R}^{M}$, $C_i \in \mathbb{N}^{M}$, threshold $\tau$
\State \textbf{Initialization:} Use existing $S_i, C_i, \hat{\mu}_{i,j} \;\forall\; j \in [M].$
\While{$P_i \geq 1$} 
\State $\mathcal{F}^{c, (i)} \gets \{j : \hat{\mu}_{i,j} \leq \tau\}$ \Comment{Set of infeasible attributes} 
\State Break if $\mathcal{F}^{c, (i)} = \emptyset$
\State $j^* \gets \min_{j \in \mathcal{F}^{c, (i)}}j$ \Comment{Select lowest index} 
\While{$\hat{\mu}_{i,j^\star} \leq \tau$}
\State Observe reward $X_{i,j^\star} \sim \nu_{i,j^\star}$
\State Update $S_{i, j^\star} , C_{i,j^\star}, \hat{\mu}_{i, j^\star}$ using \eqref{eq:upd-rule-p}.

\State $P_i \gets P_i - 1$
\State Break if $P_i = 0$ \Comment{ Break when feasibility budget exhausted} 

\EndWhile
\EndWhile
\State \Return Updated statistics $S_i, C_i$.
\end{algorithmic}
\end{algorithm}

\begin{algorithm}
\caption{$\textsc{FCSR}(T, \tau, f, g \;;\; S, C)$}
\label{alg:fcsr-p}

\begin{algorithmic}
\Require $K$ arms, $M$ attributes/arm, budget $T$, threshold $\tau$, fractions $f, g \in (0,1)$ and matrices $S, C$.
\Ensure The best feasible arm or -1.

\State \textbf{Initialize:}  Global statistics $S_{i,j}, C_{i,j}, \hat{\mu}_{i,j} \gets 0.$
\State $P_i \gets \lfloor \frac{fT}{K} \rfloor \; \forall \; i$,  $T_{\text{extra}} \gets 0$
\For{$r = 1$ to $K-1$}
    \State Define $\Delta n_r \gets n_r - n_{r-1}$ as in \eqref{eq:dnr-defp}.    
    \State $T_{\text{extra}}^{\text{arm}} \gets T_{\text{extra}} / |\mathcal{A}|$; $T_{\text{extra}} \gets 0$.

    \For{each arm $i \in \mathcal{A}$}
        \State Run $\textsc{Uniform}(i, \lfloor (1-g)\Delta n_r \rfloor + T_{\text{extra}}^{\text{arm}}; S_{i}, C_{i})$.
        \State Run $\textsc{APT}(i, \lfloor g \cdot \Delta n_r \rfloor, \tau\;;\; S_i, C_i)$.
        \State Run $\textsc{SampleUntilFeasible}(i, P_i, \tau \;;\; S_i, C_i)$.
    \EndFor
    
    \State Compute $s(i) \; \forall \; i$ using \eqref{eq:score-supm-p}.
    \State Eliminate $\mathcal{A} \gets \mathcal{A} \setminus \{ \arg \min_{i\in \mathcal{A}} s(i) \}$.
    \State $T_{\text{extra}} \gets T_{\text{extra}} + P_{\arg \min_{i\in \mathcal{A}}s(i)}$. 
\EndFor
\State \textbf{If }{the last remaining arm $i_{\text{final}} \in \mathcal{A}$ is feasible} \textbf{then }\textbf{return} $i_{\text{final}.}$
\State \textbf{Else return} $0$.
\end{algorithmic}
\end{algorithm}

\subsection{Theoretical Analysis}
\begin{theorem}[Performance of FCSR]
Let $K, M \geq 1$. Given parameters $f, g \in (0,1)$ and $\tau$. Let $c=1/32R^2$ The probability of error of $\textsc{FCSR}(T, \tau, f, g$) satisfies
\begin{equation}
\label{eq:fcsr-ub-p}
\mathbb{P}(e) \leq 3K^2\exp\left( \frac{-cT}{\log(K)H_{FC}(\mathcal{B})} \right),
\end{equation} 
for all 
\begin{align*}
T \geq \max \Big\{ 256H_{tbp}R^2\log((\log(T) + 1)M), \\ 
\frac{4K^2M}{f}, \frac{K\max\{{H_{tbp}}^{-1}, (H_{2}^R)^{-1}\}}{\log(K)} \Big\}.
\end{align*}
\end{theorem}

\textit{Proof Sketch.}
The main novelty is the sampling procedure SUF, while the overall analysis builds on the analysis of successive-rejects in \citet{Audibert2010BestArmIdentification}. We order arms by their true means and decompose the event that the optimal arm $i^\star$ is eliminated into three types: (i) $i^\star$ is incorrectly declared infeasible ($f_i$), (ii) a suboptimal feasible arm beats $i^\star$ by estimation error ($s_i$), or (iii) a risky (infeasible) arm is incorrectly declared feasible and eliminates $i^\star$ ($r_i$). A union bound over rounds gives $\sum_i \mathbb{P}(e_i)\le \sum_i \mathbb{P}(f_i)+\sum_i \mathbb{P}(s_i)+\sum_i \mathbb{P}(r_i)$. Each term is bounded separately  and the stated result is obtained. See Appendix \ref{app:fcsr} for the full proof.

Comparing the upper bound on the probability of error under FCSR \eqref{eq:fcsr-ub-p} with the lower bound \eqref{eq:lb-p}, we observe that the probability of error is characterized by the fundamental difficulty parameter $H_{FC}$ in both cases (with a $\log(K)$ factor), demonstrating the optimality of FCSR. \footnote{See Appendix \ref{app:just} for a more detailed discussion on \emph{optimality}} up to constant factors in the exponent. 

The novel sampling procedure $\textsc{SUF}$ is necessary for a much stronger upper bound on $\mathbb{P}(f_i)$ compared to the natural alternative of APT sampling. Intuitively, suppose the best arm is deemed infeasible due to the empirical mean of some attribute dipping below $\tau$, APT may concentrate samples on other less critical attributes that are deemed \emph{feasible}, simply because they are closer to the threshold. In contrast, $\textsc{SUF}$ focuses sampling on these crucial empirically infeasible attributes exclusively.

\begin{lemma}
Assuming $i^\star$ exists and $T \geq \max \{\frac{4K^2M}{f}\}$, the probability of $i^\star$ being deemed infeasible at the end of round $r \in [K-1]$ under SUF is  upper bounded by
\begin{equation}
\mathbb{P}(i^\star\in\mathcal F_r^c)\le \exp\!\Big(-\frac{fT}{16R^2\log K\,H_f}\Big).
\end{equation}
\end{lemma}

\textit{Proof Sketch and Discussion.} The proof uses the inclusion that the event the best arm is deemed infeasible at the end of round $i$ is a subset of the event that the feasibility budget runs out at the end of round $i$. We construct random variables $Z_j$ that count the number of samples that would be needed until attribute $j$ of the best arm is deemed feasible and apply stopping time analysis for a tail bound. We then aggregate this quantity over all attributes and obtain the result via a Chernoff Bound. The full proof is provided in Appendix \ref{app:suf}. 

To see the superior performance of $\textsc{SUF}$ formally, consider the case where one allocates a fraction $f$ of the per-round budget $\Delta n_r$ (defined in \eqref{eq:dnr-defp}) uniformly across all $K$ arms toward "feasibility" sampling using APT; the error probability is bounded as
\begin{equation}
\mathbb{P}_{APT}(f_i) \leq \exp \left( \frac{-f(T-K)}{\overline{\log}(K)K(K+1-i)H_{tbp}(i^\star)} \right).
\end{equation}
The above result follows from an application of the upper bound on the probability of error of APT sampling (see Lemma \ref{lem:apt-p} in Appendix \ref{app:fcsr}). The factor in the denominator in this case scales as $\mathcal{O}(K^3)$ as compared to the $\mathcal{O}(K)$ scaling for $\textsc{SUF}$.  This is due to additional factors of $K$ in the denominator and because $H_{tbp}(i^\star)$ is a sum over all attributes and while $H_f$ is only a max.

\section{NUMERICAL ANALYSIS}
We evaluate FCSR on both synthetic and real-world data. Four synthetic bandit instances are constructed to vary the hardness index  and test FCSR’s robustness in varied difficult problem instances against baseline algorithms. We also test FCSR on randomly generated bandit instances derived from the MovieLens dataset \citep{harper2015movielens} to assess performance in a low-budget, practical setting.

\subsection{Experimental Setup}
FCSR hyperparameters are set to $g = 0.3$ and $f = 0.2$, chosen empirically to approximately equalize sample allocation across mean discrimination and feasibility testing. The fraction of samples allocated to $\textsc{Uniform} = 0.8 \times 0.7 = 0.56$. The fraction allocated to $\textsc{SUF} = 0.2$, $\textsc{APT} = 0.8 \times 0.3 = 0.24$. 

Synthetic experiments use bandit instances with $K = 10$ arms and $M = 5$ attributes, run for $N = 2000$ iterations. Attributes follow $\nu_{ij} \sim \mathcal{N}(\mu_{ij}, 0.3)$, where $\mathcal{N}(0,1)$ denotes a normal distribution with mean $0$ and variance $1$.  
Algorithms are evaluated using the probability of error, $\mathbb{P}(\text{error})$, defined as the fraction of runs where an algorithm fails to identify the true best feasible arm $i^\star$.  We plot the natural logarithm of the error probability, $\log(\mathbb{P}(\text{error}))$, against the total budget $T$. Reported values are averaged over the $N = 2000$ trials. 
\subsection{Baselines}

\paragraph{Uniform Sampling (US).} The total budget is divided equally across all arm-attribute pairs. Empirical means and the set of empirically feasible arms are computed. Upon exhaustion of the budget, the arm with the highest empirical overall mean is selected; if no arm is feasible, the algorithm returns 0 by convention.

\paragraph{Successive Rejects (SR).} We use a variant of SR that is similar to the Infeasibility First scheme of \cite{Faizal2022ConstrainedPE}. Samples are allocated using the Successive Rejects schedule. Arms are eliminated similarly as in FCSR, using the scoring rule in \eqref{eq:score-supm-p}. The final arm is flagged as optimal if feasible, otherwise 0 is returned.

\paragraph{Explore-then-Commit (ETC).} A two-stage variant of the classic Explore-then-Commit algorithm. In the first stage, a fixed fraction of the total budget is allocated to uniformly sample all attributes. Arms are ranked as in SR, and the top-$M$ arms form a candidate set. The remaining budget is then uniformly allocated among attributes of these candidates. The highest-ranked arm at the end is returned.


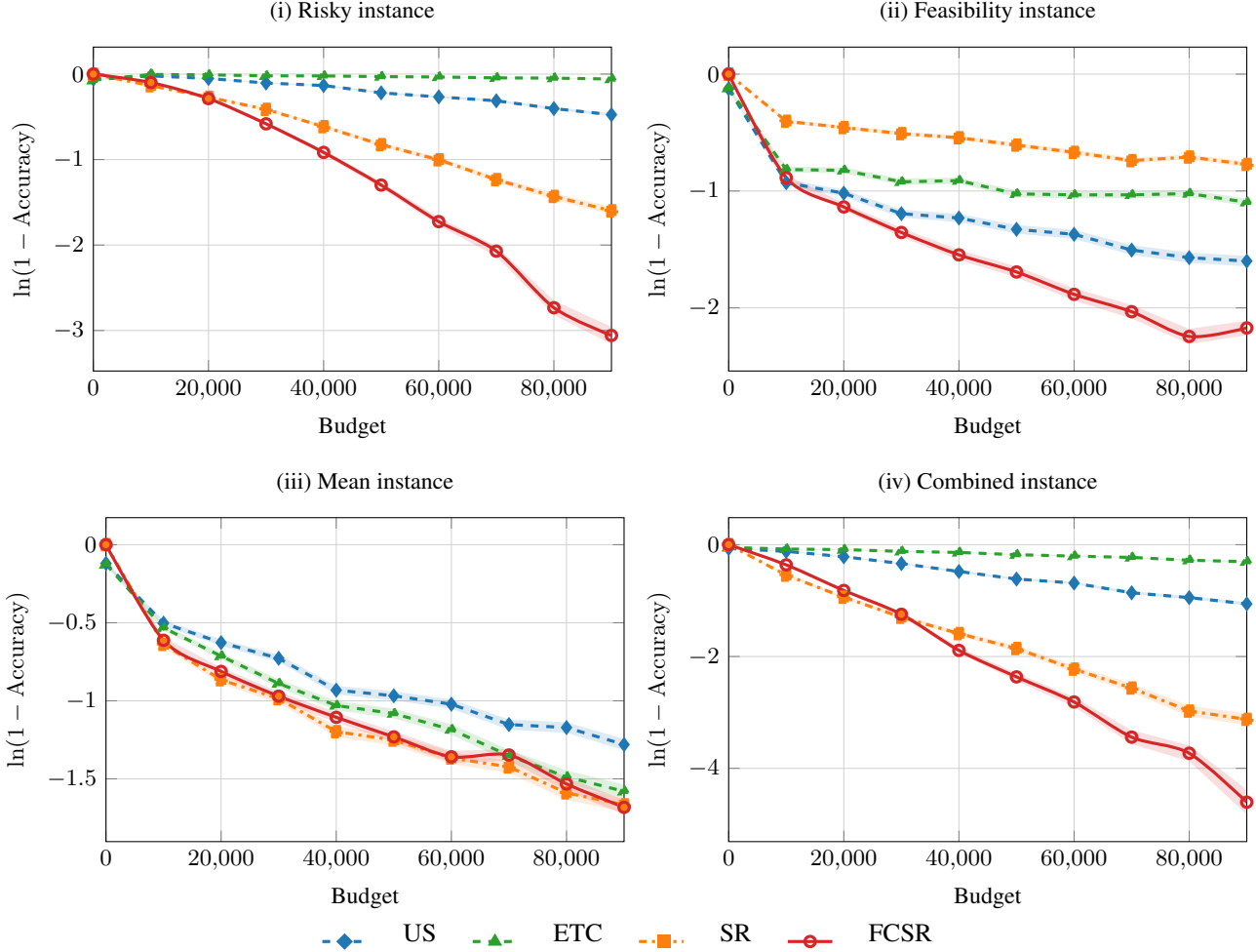
\begin{figure*}[!t]
  \centering

  \pgfplotsset{
    mygrid/.style={
      grid=both,
      major grid style={line width=0.22pt, draw=black!18},
      minor grid style={line width=0.08pt, draw=black!8}
    }
  }

  \noindent
  \begin{minipage}[t]{0.5\textwidth}
    \centering
    \begin{tikzpicture}
      \begin{axis}[
        width=\linewidth,
        height=6cm,
        title={ (i) Risky instance},
        xlabel={Budget},
        ylabel={$\ln(1-\mathrm{Accuracy})$},
        xmin=0, xmax=90000,
        xtick={0,20000,40000,60000,80000},
        mygrid,
        legend pos=south west,
        legend cell align=left
      ]
        \addplot[name path=us, color=colUS, very thick, dashed,  mark=diamond*, mark options = {solid}] coordinates {
          (0,   {ln(1-0.057)}) (10000,{ln(1-0.021)}) (20000,{ln(1-0.050)})
          (30000,{ln(1-0.099)}) (40000,{ln(1-0.125)}) (50000,{ln(1-0.197)})
          (60000,{ln(1-0.234)}) (70000,{ln(1-0.268)}) (80000,{ln(1-0.332)})
          (90000,{ln(1-0.378)})
        };
        \addplot[name path=usU, draw=none] coordinates {
          (0,{ln(1-0.057)+sqrt(0.057/(2000*(1-0.057)))})
          (10000,{ln(1-0.021)+sqrt(0.021/(2000*(1-0.021)))})
          (20000,{ln(1-0.050)+sqrt(0.050/(2000*(1-0.050)))})
          (30000,{ln(1-0.099)+sqrt(0.099/(2000*(1-0.099)))})
          (40000,{ln(1-0.125)+sqrt(0.125/(2000*(1-0.125)))})
          (50000,{ln(1-0.197)+sqrt(0.197/(2000*(1-0.197)))})
          (60000,{ln(1-0.234)+sqrt(0.234/(2000*(1-0.234)))})
          (70000,{ln(1-0.268)+sqrt(0.268/(2000*(1-0.268)))})
          (80000,{ln(1-0.332)+sqrt(0.332/(2000*(1-0.332)))})
          (90000,{ln(1-0.378)+sqrt(0.378/(2000*(1-0.378)))})
        };
        \addplot[name path=usL, draw=none] coordinates {
          (0,{ln(1-0.057)-sqrt(0.057/(2000*(1-0.057)))})
          (10000,{ln(1-0.021)-sqrt(0.021/(2000*(1-0.021)))})
          (20000,{ln(1-0.050)-sqrt(0.050/(2000*(1-0.050)))})
          (30000,{ln(1-0.099)-sqrt(0.099/(2000*(1-0.099)))})
          (40000,{ln(1-0.125)-sqrt(0.125/(2000*(1-0.125)))})
          (50000,{ln(1-0.197)-sqrt(0.197/(2000*(1-0.197)))})
          (60000,{ln(1-0.234)-sqrt(0.234/(2000*(1-0.234)))})
          (70000,{ln(1-0.268)-sqrt(0.268/(2000*(1-0.268)))})
          (80000,{ln(1-0.332)-sqrt(0.332/(2000*(1-0.332)))})
          (90000,{ln(1-0.378)-sqrt(0.378/(2000*(1-0.378)))})
        };
        \addplot[fill=colUS, fill opacity=0.15] fill between[of=usU and usL];
        \addplot[name path=etc, color=colETC, very thick, dashed, mark=triangle*] coordinates {
          (0,   {ln(1-0.070)}) (10000,{ln(1-0.004)}) (20000,{ln(1-0.009)})
          (30000,{ln(1-0.020)}) (40000,{ln(1-0.021)}) (50000,{ln(1-0.029)})
          (60000,{ln(1-0.033)}) (70000,{ln(1-0.043)}) (80000,{ln(1-0.046)})
          (90000,{ln(1-0.056)})
        };
        \addplot[name path=etcU, draw=none] coordinates {
          (0,{ln(1-0.070)+sqrt(0.070/(2000*(1-0.070)))})
          (10000,{ln(1-0.004)+sqrt(0.004/(2000*(1-0.004)))})
          (20000,{ln(1-0.009)+sqrt(0.009/(2000*(1-0.009)))})
          (30000,{ln(1-0.020)+sqrt(0.020/(2000*(1-0.020)))})
          (40000,{ln(1-0.021)+sqrt(0.021/(2000*(1-0.021)))})
          (50000,{ln(1-0.029)+sqrt(0.029/(2000*(1-0.029)))})
          (60000,{ln(1-0.033)+sqrt(0.033/(2000*(1-0.033)))})
          (70000,{ln(1-0.043)+sqrt(0.043/(2000*(1-0.043)))})
          (80000,{ln(1-0.046)+sqrt(0.046/(2000*(1-0.046)))})
          (90000,{ln(1-0.056)+sqrt(0.056/(2000*(1-0.056)))})
        };
        \addplot[name path=etcL, draw=none] coordinates {
          (0,{ln(1-0.070)-sqrt(0.070/(2000*(1-0.070)))})
          (10000,{ln(1-0.004)-sqrt(0.004/(2000*(1-0.004)))})
          (20000,{ln(1-0.009)-sqrt(0.009/(2000*(1-0.009)))})
          (30000,{ln(1-0.020)-sqrt(0.020/(2000*(1-0.020)))})
          (40000,{ln(1-0.021)-sqrt(0.021/(2000*(1-0.021)))})
          (50000,{ln(1-0.029)-sqrt(0.029/(2000*(1-0.029)))})
          (60000,{ln(1-0.033)-sqrt(0.033/(2000*(1-0.033)))})
          (70000,{ln(1-0.043)-sqrt(0.043/(2000*(1-0.043)))})
          (80000,{ln(1-0.046)-sqrt(0.046/(2000*(1-0.046)))})
          (90000,{ln(1-0.056)-sqrt(0.056/(2000*(1-0.056)))})
        };
        \addplot[fill=colETC, fill opacity=0.15] fill between[of=etcU and etcL];
        \addplot[name path=sr, color=colSR, very thick, dashdotted, mark=square*] coordinates {
          (0,   {ln(1-0.000)}) (10000,{ln(1-0.127)}) (20000,{ln(1-0.237)})
          (30000,{ln(1-0.338)}) (40000,{ln(1-0.459)}) (50000,{ln(1-0.562)})
          (60000,{ln(1-0.633)}) (70000,{ln(1-0.709)}) (80000,{ln(1-0.761)})
          (90000,{ln(1-0.798)})
        };
        \addplot[name path=srU, draw=none] coordinates {
          (0,{ln(1-0.000)+0})
          (10000,{ln(1-0.127)+sqrt(0.127/(2000*(1-0.127)))})
          (20000,{ln(1-0.237)+sqrt(0.237/(2000*(1-0.237)))})
          (30000,{ln(1-0.338)+sqrt(0.338/(2000*(1-0.338)))})
          (40000,{ln(1-0.459)+sqrt(0.459/(2000*(1-0.459)))})
          (50000,{ln(1-0.562)+sqrt(0.562/(2000*(1-0.562)))})
          (60000,{ln(1-0.633)+sqrt(0.633/(2000*(1-0.633)))})
          (70000,{ln(1-0.709)+sqrt(0.709/(2000*(1-0.709)))})
          (80000,{ln(1-0.761)+sqrt(0.761/(2000*(1-0.761)))})
          (90000,{ln(1-0.798)+sqrt(0.798/(2000*(1-0.798)))})
        };
        \addplot[name path=srL, draw=none] coordinates {
          (0,{ln(1-0.000)-0})
          (10000,{ln(1-0.127)-sqrt(0.127/(2000*(1-0.127)))})
          (20000,{ln(1-0.237)-sqrt(0.237/(2000*(1-0.237)))})
          (30000,{ln(1-0.338)-sqrt(0.338/(2000*(1-0.338)))})
          (40000,{ln(1-0.459)-sqrt(0.459/(2000*(1-0.459)))})
          (50000,{ln(1-0.562)-sqrt(0.562/(2000*(1-0.562)))})
          (60000,{ln(1-0.633)-sqrt(0.633/(2000*(1-0.633)))})
          (70000,{ln(1-0.709)-sqrt(0.709/(2000*(1-0.709)))})
          (80000,{ln(1-0.761)-sqrt(0.761/(2000*(1-0.761)))})
          (90000,{ln(1-0.798)-sqrt(0.798/(2000*(1-0.798)))})
        };
        \addplot[fill=colSR, fill opacity=0.15] fill between[of=srU and srL];
        \addplot[name path=fcsr, color=colFCSR, very thick, smooth, mark=o, mark options={solid}] coordinates {
          (0,   {ln(1-0.000)}) (10000,{ln(1-0.094)}) (20000,{ln(1-0.249)})
          (30000,{ln(1-0.441)}) (40000,{ln(1-0.600)}) (50000,{ln(1-0.727)})
          (60000,{ln(1-0.822)}) (70000,{ln(1-0.874)}) (80000,{ln(1-0.935)})
          (90000,{ln(1-0.953)})
        };
        \addplot[name path=fcsrU, draw=none] coordinates {
          (0,{ln(1-0.000)+0})
          (10000,{ln(1-0.094)+sqrt(0.094/(2000*(1-0.094)))})
          (20000,{ln(1-0.249)+sqrt(0.249/(2000*(1-0.249)))})
          (30000,{ln(1-0.441)+sqrt(0.441/(2000*(1-0.441)))})
          (40000,{ln(1-0.600)+sqrt(0.600/(2000*(1-0.600)))})
          (50000,{ln(1-0.727)+sqrt(0.727/(2000*(1-0.727)))})
          (60000,{ln(1-0.822)+sqrt(0.822/(2000*(1-0.822)))})
          (70000,{ln(1-0.874)+sqrt(0.874/(2000*(1-0.874)))})
          (80000,{ln(1-0.935)+sqrt(0.935/(2000*(1-0.935)))})
          (90000,{ln(1-0.953)+sqrt(0.953/(2000*(1-0.953)))})
        };
        \addplot[name path=fcsrL, draw=none] coordinates {
          (0,{ln(1-0.000)-0})
          (10000,{ln(1-0.094)-sqrt(0.094/(2000*(1-0.094)))})
          (20000,{ln(1-0.249)-sqrt(0.249/(2000*(1-0.249)))})
          (30000,{ln(1-0.441)-sqrt(0.441/(2000*(1-0.441)))})
          (40000,{ln(1-0.600)-sqrt(0.600/(2000*(1-0.600)))})
          (50000,{ln(1-0.727)-sqrt(0.727/(2000*(1-0.727)))})
          (60000,{ln(1-0.822)-sqrt(0.822/(2000*(1-0.822)))})
          (70000,{ln(1-0.874)-sqrt(0.874/(2000*(1-0.874)))})
          (80000,{ln(1-0.935)-sqrt(0.935/(2000*(1-0.935)))})
          (90000,{ln(1-0.953)-sqrt(0.953/(2000*(1-0.953)))})
        };
        \addplot[fill=colFCSR, fill opacity=0.15] fill between[of=fcsrU and fcsrL];
      \end{axis}
    \end{tikzpicture}
    \label{fig:synthetic_risky}
  \end{minipage}\hfill%
  \begin{minipage}[t]{0.5\textwidth}
    \centering
    \begin{tikzpicture}
      \begin{axis}[
        width=\linewidth,
        height=6cm,
        title={(ii) Feasibility instance },
        xlabel={Budget},
        ylabel={$\ln(1-\mathrm{Accuracy})$},
        xmin=0, xmax=90000,
        xtick={0,20000,40000,60000,80000},
        mygrid
      ]
        \addplot[name path=us2, color=colUS, very thick, dashed, mark=diamond*, mark options = {solid}] coordinates {
          (0,   {ln(1-0.118)}) (10000,{ln(1-0.605)}) (20000,{ln(1-0.639)})
          (30000,{ln(1-0.697)}) (40000,{ln(1-0.708)}) (50000,{ln(1-0.735)})
          (60000,{ln(1-0.746)}) (70000,{ln(1-0.778)}) (80000,{ln(1-0.792)})
          (90000,{ln(1-0.798)})
        };
        \addplot[name path=us2U, draw=none] coordinates {
          (0,{ln(1-0.118)+sqrt(0.118/(2000*(1-0.118)))})
          (10000,{ln(1-0.605)+sqrt(0.605/(2000*(1-0.605)))})
          (20000,{ln(1-0.639)+sqrt(0.639/(2000*(1-0.639)))})
          (30000,{ln(1-0.697)+sqrt(0.697/(2000*(1-0.697)))})
          (40000,{ln(1-0.708)+sqrt(0.708/(2000*(1-0.708)))})
          (50000,{ln(1-0.735)+sqrt(0.735/(2000*(1-0.735)))})
          (60000,{ln(1-0.746)+sqrt(0.746/(2000*(1-0.746)))})
          (70000,{ln(1-0.778)+sqrt(0.778/(2000*(1-0.778)))})
          (80000,{ln(1-0.792)+sqrt(0.792/(2000*(1-0.792)))})
          (90000,{ln(1-0.798)+sqrt(0.798/(2000*(1-0.798)))})
        };
        \addplot[name path=us2L, draw=none] coordinates {
          (0,{ln(1-0.118)-sqrt(0.118/(2000*(1-0.118)))})
          (10000,{ln(1-0.605)-sqrt(0.605/(2000*(1-0.605)))})
          (20000,{ln(1-0.639)-sqrt(0.639/(2000*(1-0.639)))})
          (30000,{ln(1-0.697)-sqrt(0.697/(2000*(1-0.697)))})
          (40000,{ln(1-0.708)-sqrt(0.708/(2000*(1-0.708)))})
          (50000,{ln(1-0.735)-sqrt(0.735/(2000*(1-0.735)))})
          (60000,{ln(1-0.746)-sqrt(0.746/(2000*(1-0.746)))})
          (70000,{ln(1-0.778)-sqrt(0.778/(2000*(1-0.778)))})
          (80000,{ln(1-0.792)-sqrt(0.792/(2000*(1-0.792)))})
          (90000,{ln(1-0.798)-sqrt(0.798/(2000*(1-0.798)))})
        };
        \addplot[fill=colUS, fill opacity=0.15] fill between[of=us2U and us2L];
        \addplot[name path=etc2, color=colETC, very thick, dashed, mark=triangle*] coordinates {
          (0,   {ln(1-0.113)}) (10000,{ln(1-0.558)}) (20000,{ln(1-0.561)})
          (30000,{ln(1-0.602)}) (40000,{ln(1-0.598)}) (50000,{ln(1-0.641)})
          (60000,{ln(1-0.644)}) (70000,{ln(1-0.644)}) (80000,{ln(1-0.640)})
          (90000,{ln(1-0.666)})
        };
        \addplot[name path=etc2U, draw=none] coordinates {
          (0,{ln(1-0.113)+sqrt(0.113/(2000*(1-0.113)))})
          (10000,{ln(1-0.558)+sqrt(0.558/(2000*(1-0.558)))})
          (20000,{ln(1-0.561)+sqrt(0.561/(2000*(1-0.561)))})
          (30000,{ln(1-0.602)+sqrt(0.602/(2000*(1-0.602)))})
          (40000,{ln(1-0.598)+sqrt(0.598/(2000*(1-0.598)))})
          (50000,{ln(1-0.641)+sqrt(0.641/(2000*(1-0.641)))})
          (60000,{ln(1-0.644)+sqrt(0.644/(2000*(1-0.644)))})
          (70000,{ln(1-0.644)+sqrt(0.644/(2000*(1-0.644)))})
          (80000,{ln(1-0.640)+sqrt(0.640/(2000*(1-0.640)))})
          (90000,{ln(1-0.666)+sqrt(0.666/(2000*(1-0.666)))})
        };
        \addplot[name path=etc2L, draw=none] coordinates {
          (0,{ln(1-0.113)-sqrt(0.113/(2000*(1-0.113)))})
          (10000,{ln(1-0.558)-sqrt(0.558/(2000*(1-0.558)))})
          (20000,{ln(1-0.561)-sqrt(0.561/(2000*(1-0.561)))})
          (30000,{ln(1-0.602)-sqrt(0.602/(2000*(1-0.602)))})
          (40000,{ln(1-0.598)-sqrt(0.598/(2000*(1-0.598)))})
          (50000,{ln(1-0.641)-sqrt(0.641/(2000*(1-0.641)))})
          (60000,{ln(1-0.644)-sqrt(0.644/(2000*(1-0.644)))})
          (70000,{ln(1-0.644)-sqrt(0.644/(2000*(1-0.644)))})
          (80000,{ln(1-0.640)-sqrt(0.640/(2000*(1-0.640)))})
          (90000,{ln(1-0.666)-sqrt(0.666/(2000*(1-0.666)))})
        };
        \addplot[fill=colETC, fill opacity=0.15] fill between[of=etc2U and etc2L];
        \addplot[name path=sr2, color=colSR, very thick, dashdotted, mark=square*] coordinates {
          (0,   {ln(1-0.000)}) (10000,{ln(1-0.332)}) (20000,{ln(1-0.367)})
          (30000,{ln(1-0.399)}) (40000,{ln(1-0.420)}) (50000,{ln(1-0.455)})
          (60000,{ln(1-0.488)}) (70000,{ln(1-0.523)}) (80000,{ln(1-0.508)})
          (90000,{ln(1-0.538)})
        };
        \addplot[name path=sr2U, draw=none] coordinates {
          (0,{ln(1-0.000)+0})
          (10000,{ln(1-0.332)+sqrt(0.332/(2000*(1-0.332)))})
          (20000,{ln(1-0.367)+sqrt(0.367/(2000*(1-0.367)))})
          (30000,{ln(1-0.399)+sqrt(0.399/(2000*(1-0.399)))})
          (40000,{ln(1-0.420)+sqrt(0.420/(2000*(1-0.420)))})
          (50000,{ln(1-0.455)+sqrt(0.455/(2000*(1-0.455)))})
          (60000,{ln(1-0.488)+sqrt(0.488/(2000*(1-0.488)))})
          (70000,{ln(1-0.523)+sqrt(0.523/(2000*(1-0.523)))})
          (80000,{ln(1-0.508)+sqrt(0.508/(2000*(1-0.508)))})
          (90000,{ln(1-0.538)+sqrt(0.538/(2000*(1-0.538)))})
        };
        \addplot[name path=sr2L, draw=none] coordinates {
          (0,{ln(1-0.000)-0})
          (10000,{ln(1-0.332)-sqrt(0.332/(2000*(1-0.332)))})
          (20000,{ln(1-0.367)-sqrt(0.367/(2000*(1-0.367)))})
          (30000,{ln(1-0.399)-sqrt(0.399/(2000*(1-0.399)))})
          (40000,{ln(1-0.420)-sqrt(0.420/(2000*(1-0.420)))})
          (50000,{ln(1-0.455)-sqrt(0.455/(2000*(1-0.455)))})
          (60000,{ln(1-0.488)-sqrt(0.488/(2000*(1-0.488)))})
          (70000,{ln(1-0.523)-sqrt(0.523/(2000*(1-0.523)))})
          (80000,{ln(1-0.508)-sqrt(0.508/(2000*(1-0.508)))})
          (90000,{ln(1-0.538)-sqrt(0.538/(2000*(1-0.538)))})
        };
        \addplot[fill=colSR, fill opacity=0.15] fill between[of=sr2U and sr2L];
        \addplot[name path=fcsr2, color=colFCSR, very thick, smooth, mark=o, mark options={solid}] coordinates {
          (0,   {ln(1-0.000)}) (10000,{ln(1-0.588)}) (20000,{ln(1-0.679)})
          (30000,{ln(1-0.742)}) (40000,{ln(1-0.787)}) (50000,{ln(1-0.816)})
          (60000,{ln(1-0.848)}) (70000,{ln(1-0.869)}) (80000,{ln(1-0.894)})
          (90000,{ln(1-0.886)})
        };
        \addplot[name path=fcsr2U, draw=none] coordinates {
          (0,{ln(1-0.000)+0})
          (10000,{ln(1-0.588)+sqrt(0.588/(2000*(1-0.588)))})
          (20000,{ln(1-0.679)+sqrt(0.679/(2000*(1-0.679)))})
          (30000,{ln(1-0.742)+sqrt(0.742/(2000*(1-0.742)))})
          (40000,{ln(1-0.787)+sqrt(0.787/(2000*(1-0.787)))})
          (50000,{ln(1-0.816)+sqrt(0.816/(2000*(1-0.816)))})
          (60000,{ln(1-0.848)+sqrt(0.848/(2000*(1-0.848)))})
          (70000,{ln(1-0.869)+sqrt(0.869/(2000*(1-0.869)))})
          (80000,{ln(1-0.894)+sqrt(0.894/(2000*(1-0.894)))})
          (90000,{ln(1-0.886)+sqrt(0.886/(2000*(1-0.886)))})
        };
        \addplot[name path=fcsr2L, draw=none] coordinates {
          (0,{ln(1-0.000)-0})
          (10000,{ln(1-0.588)-sqrt(0.588/(2000*(1-0.588)))})
          (20000,{ln(1-0.679)-sqrt(0.679/(2000*(1-0.679)))})
          (30000,{ln(1-0.742)-sqrt(0.742/(2000*(1-0.742)))})
          (40000,{ln(1-0.787)-sqrt(0.787/(2000*(1-0.787)))})
          (50000,{ln(1-0.816)-sqrt(0.816/(2000*(1-0.816)))})
          (60000,{ln(1-0.848)-sqrt(0.848/(2000*(1-0.848)))})
          (70000,{ln(1-0.869)-sqrt(0.869/(2000*(1-0.869)))})
          (80000,{ln(1-0.894)-sqrt(0.894/(2000*(1-0.894)))})
          (90000,{ln(1-0.886)-sqrt(0.886/(2000*(1-0.886)))})
        };
        \addplot[fill=colFCSR, fill opacity=0.15] fill between[of=fcsr2U and fcsr2L];
      \end{axis}
    \end{tikzpicture}
    \label{fig:synthetic_feas}
  \end{minipage}

  \vspace{6pt}

  \noindent
  \begin{minipage}[t]{0.5\textwidth}
    \centering
    \begin{tikzpicture}
      \begin{axis}[
        width=\linewidth,
        height=6cm, 
        title={(iii) Mean instance },
        xlabel={Budget},
        ylabel={$\ln(1-\mathrm{Accuracy})$},
        xmin=0, xmax=90000,
        xtick={0,20000,40000,60000,80000},
        mygrid
      ]
        \addplot[name path=us3, color=colUS, very thick, dashed, mark=diamond*, mark options = {solid}] coordinates {
          (0,   {ln(1-0.113)}) (10000,{ln(1-0.395)}) (20000,{ln(1-0.466)})
          (30000,{ln(1-0.517)}) (40000,{ln(1-0.606)}) (50000,{ln(1-0.620)})
          (60000,{ln(1-0.640)}) (70000,{ln(1-0.684)}) (80000,{ln(1-0.690)})
          (90000,{ln(1-0.722)})
        };
        \addplot[name path=us3U, draw=none] coordinates {
          (0,{ln(1-0.113)+sqrt(0.113/(2000*(1-0.113)))})
          (10000,{ln(1-0.395)+sqrt(0.395/(2000*(1-0.395)))})
          (20000,{ln(1-0.466)+sqrt(0.466/(2000*(1-0.466)))})
          (30000,{ln(1-0.517)+sqrt(0.517/(2000*(1-0.517)))})
          (40000,{ln(1-0.606)+sqrt(0.606/(2000*(1-0.606)))})
          (50000,{ln(1-0.620)+sqrt(0.620/(2000*(1-0.620)))})
          (60000,{ln(1-0.640)+sqrt(0.640/(2000*(1-0.640)))})
          (70000,{ln(1-0.684)+sqrt(0.684/(2000*(1-0.684)))})
          (80000,{ln(1-0.690)+sqrt(0.690/(2000*(1-0.690)))})
          (90000,{ln(1-0.722)+sqrt(0.722/(2000*(1-0.722)))})
        };
        \addplot[name path=us3L, draw=none] coordinates {
          (0,{ln(1-0.113)-sqrt(0.113/(2000*(1-0.113)))})
          (10000,{ln(1-0.395)-sqrt(0.395/(2000*(1-0.395)))})
          (20000,{ln(1-0.466)-sqrt(0.466/(2000*(1-0.466)))})
          (30000,{ln(1-0.517)-sqrt(0.517/(2000*(1-0.517)))})
          (40000,{ln(1-0.606)-sqrt(0.606/(2000*(1-0.606)))})
          (50000,{ln(1-0.620)-sqrt(0.620/(2000*(1-0.620)))})
          (60000,{ln(1-0.640)-sqrt(0.640/(2000*(1-0.640)))})
          (70000,{ln(1-0.684)-sqrt(0.684/(2000*(1-0.684)))})
          (80000,{ln(1-0.690)-sqrt(0.690/(2000*(1-0.690)))})
          (90000,{ln(1-0.722)-sqrt(0.722/(2000*(1-0.722)))})
        };
        \addplot[fill=colUS, fill opacity=0.15] fill between[of=us3U and us3L];
        \addplot[name path=etc3, color=colETC, very thick, dashed, mark=triangle*] coordinates {
          (0,   {ln(1-0.119)}) (10000,{ln(1-0.413)}) (20000,{ln(1-0.510)})
          (30000,{ln(1-0.589)}) (40000,{ln(1-0.643)}) (50000,{ln(1-0.661)})
          (60000,{ln(1-0.694)}) (70000,{ln(1-0.741)}) (80000,{ln(1-0.774)})
          (90000,{ln(1-0.794)})
        };
        \addplot[name path=etc3U, draw=none] coordinates {
          (0,{ln(1-0.119)+sqrt(0.119/(2000*(1-0.119)))})
          (10000,{ln(1-0.413)+sqrt(0.413/(2000*(1-0.413)))})
          (20000,{ln(1-0.510)+sqrt(0.510/(2000*(1-0.510)))})
          (30000,{ln(1-0.589)+sqrt(0.589/(2000*(1-0.589)))})
          (40000,{ln(1-0.643)+sqrt(0.643/(2000*(1-0.643)))})
          (50000,{ln(1-0.661)+sqrt(0.661/(2000*(1-0.661)))})
          (60000,{ln(1-0.694)+sqrt(0.694/(2000*(1-0.694)))})
          (70000,{ln(1-0.741)+sqrt(0.741/(2000*(1-0.741)))})
          (80000,{ln(1-0.774)+sqrt(0.774/(2000*(1-0.774)))})
          (90000,{ln(1-0.794)+sqrt(0.794/(2000*(1-0.794)))})
        };
        \addplot[name path=etc3L, draw=none] coordinates {
          (0,{ln(1-0.119)-sqrt(0.119/(2000*(1-0.119)))})
          (10000,{ln(1-0.413)-sqrt(0.413/(2000*(1-0.413)))})
          (20000,{ln(1-0.510)-sqrt(0.510/(2000*(1-0.510)))})
          (30000,{ln(1-0.589)-sqrt(0.589/(2000*(1-0.589)))})
          (40000,{ln(1-0.643)-sqrt(0.643/(2000*(1-0.643)))})
          (50000,{ln(1-0.661)-sqrt(0.661/(2000*(1-0.661)))})
          (60000,{ln(1-0.694)-sqrt(0.694/(2000*(1-0.694)))})
          (70000,{ln(1-0.741)-sqrt(0.741/(2000*(1-0.741)))})
          (80000,{ln(1-0.774)-sqrt(0.774/(2000*(1-0.774)))})
          (90000,{ln(1-0.794)-sqrt(0.794/(2000*(1-0.794)))})
        };
        \addplot[fill=colETC, fill opacity=0.15] fill between[of=etc3U and etc3L];
        \addplot[name path=sr3, color=colSR, very thick, dashdotted, mark=square*] coordinates {
          (0,   {ln(1-0.000)}) (10000,{ln(1-0.472)}) (20000,{ln(1-0.579)})
          (30000,{ln(1-0.627)}) (40000,{ln(1-0.698)}) (50000,{ln(1-0.714)})
          (60000,{ln(1-0.745)}) (70000,{ln(1-0.759)}) (80000,{ln(1-0.796)})
          (90000,{ln(1-0.811)})
        };
        \addplot[name path=sr3U, draw=none] coordinates {
          (0,{ln(1-0.000)+0})
          (10000,{ln(1-0.472)+sqrt(0.472/(2000*(1-0.472)))})
          (20000,{ln(1-0.579)+sqrt(0.579/(2000*(1-0.579)))})
          (30000,{ln(1-0.627)+sqrt(0.627/(2000*(1-0.627)))})
          (40000,{ln(1-0.698)+sqrt(0.698/(2000*(1-0.698)))})
          (50000,{ln(1-0.714)+sqrt(0.714/(2000*(1-0.714)))})
          (60000,{ln(1-0.745)+sqrt(0.745/(2000*(1-0.745)))})
          (70000,{ln(1-0.759)+sqrt(0.759/(2000*(1-0.759)))})
          (80000,{ln(1-0.796)+sqrt(0.796/(2000*(1-0.796)))})
          (90000,{ln(1-0.811)+sqrt(0.811/(2000*(1-0.811)))})
        };
        \addplot[name path=sr3L, draw=none] coordinates {
          (0,{ln(1-0.000)-0})
          (10000,{ln(1-0.472)-sqrt(0.472/(2000*(1-0.472)))})
          (20000,{ln(1-0.579)-sqrt(0.579/(2000*(1-0.579)))})
          (30000,{ln(1-0.627)-sqrt(0.627/(2000*(1-0.627)))})
          (40000,{ln(1-0.698)-sqrt(0.698/(2000*(1-0.698)))})
          (50000,{ln(1-0.714)-sqrt(0.714/(2000*(1-0.714)))})
          (60000,{ln(1-0.745)-sqrt(0.745/(2000*(1-0.745)))})
          (70000,{ln(1-0.759)-sqrt(0.759/(2000*(1-0.759)))})
          (80000,{ln(1-0.796)-sqrt(0.796/(2000*(1-0.796)))})
          (90000,{ln(1-0.811)-sqrt(0.811/(2000*(1-0.811)))})
        };
        \addplot[fill=colSR, fill opacity=0.15] fill between[of=sr3U and sr3L];
        \addplot[name path=fcsr3, color=colFCSR, very thick, smooth, mark=o, mark options={solid}] coordinates {
          (0,   {ln(1-0.000)}) (10000,{ln(1-0.458)}) (20000,{ln(1-0.556)})
          (30000,{ln(1-0.621)}) (40000,{ln(1-0.669)}) (50000,{ln(1-0.708)})
          (60000,{ln(1-0.743)}) (70000,{ln(1-0.740)}) (80000,{ln(1-0.784)})
          (90000,{ln(1-0.814)})
        };
        \addplot[name path=fcsr3U, draw=none] coordinates {
          (0,{ln(1-0.000)+0})
          (10000,{ln(1-0.458)+sqrt(0.458/(2000*(1-0.458)))})
          (20000,{ln(1-0.556)+sqrt(0.556/(2000*(1-0.556)))})
          (30000,{ln(1-0.621)+sqrt(0.621/(2000*(1-0.621)))})
          (40000,{ln(1-0.669)+sqrt(0.669/(2000*(1-0.669)))})
          (50000,{ln(1-0.708)+sqrt(0.708/(2000*(1-0.708)))})
          (60000,{ln(1-0.743)+sqrt(0.743/(2000*(1-0.743)))})
          (70000,{ln(1-0.740)+sqrt(0.740/(2000*(1-0.740)))})
          (80000,{ln(1-0.784)+sqrt(0.784/(2000*(1-0.784)))})
          (90000,{ln(1-0.814)+sqrt(0.814/(2000*(1-0.814)))})
        };
        \addplot[name path=fcsr3L, draw=none] coordinates {
          (0,{ln(1-0.000)-0})
          (10000,{ln(1-0.458)-sqrt(0.458/(2000*(1-0.458)))})
          (20000,{ln(1-0.556)-sqrt(0.556/(2000*(1-0.556)))})
          (30000,{ln(1-0.621)-sqrt(0.621/(2000*(1-0.621)))})
          (40000,{ln(1-0.669)-sqrt(0.669/(2000*(1-0.669)))})
          (50000,{ln(1-0.708)-sqrt(0.708/(2000*(1-0.708)))})
          (60000,{ln(1-0.743)-sqrt(0.743/(2000*(1-0.743)))})
          (70000,{ln(1-0.740)-sqrt(0.740/(2000*(1-0.740)))})
          (80000,{ln(1-0.784)-sqrt(0.784/(2000*(1-0.784)))})
          (90000,{ln(1-0.814)-sqrt(0.814/(2000*(1-0.814)))})
        };
        \addplot[fill=colFCSR, fill opacity=0.15] fill between[of=fcsr3U and fcsr3L];
      \end{axis}
    \end{tikzpicture}
    \label{fig:synthetic_mean}
  \end{minipage}\hfill%
  \begin{minipage}[t]{0.5\textwidth}
    \centering
    \begin{tikzpicture}
      \begin{axis}[
        width=\linewidth,
        height=6cm,
        title={(iv) Combined instance},
        xlabel={Budget},
        ylabel={$\ln(1-\mathrm{Accuracy})$},
        xmin=0, xmax=90000,
        xtick={0,20000,40000,60000,80000},
        mygrid
      ]
        \addplot[name path=us4, color=colUS, very thick, dashed, mark=diamond*, mark options = {solid}] coordinates {
          (0,   {ln(1-0.066)}) (10000,{ln(1-0.113)}) (20000,{ln(1-0.194)})
          (30000,{ln(1-0.287)}) (40000,{ln(1-0.381)}) (50000,{ln(1-0.459)})
          (60000,{ln(1-0.497)}) (70000,{ln(1-0.577)}) (80000,{ln(1-0.612)})
          (90000,{ln(1-0.653)})
        };
        \addplot[name path=us4U, draw=none] coordinates {
          (0,{ln(1-0.066)+sqrt(0.066/(2000*(1-0.066)))})
          (10000,{ln(1-0.113)+sqrt(0.113/(2000*(1-0.113)))})
          (20000,{ln(1-0.194)+sqrt(0.194/(2000*(1-0.194)))})
          (30000,{ln(1-0.287)+sqrt(0.287/(2000*(1-0.287)))})
          (40000,{ln(1-0.381)+sqrt(0.381/(2000*(1-0.381)))})
          (50000,{ln(1-0.459)+sqrt(0.459/(2000*(1-0.459)))})
          (60000,{ln(1-0.497)+sqrt(0.497/(2000*(1-0.497)))})
          (70000,{ln(1-0.577)+sqrt(0.577/(2000*(1-0.577)))})
          (80000,{ln(1-0.612)+sqrt(0.612/(2000*(1-0.612)))})
          (90000,{ln(1-0.653)+sqrt(0.653/(2000*(1-0.653)))})
        };
        \addplot[name path=us4L, draw=none] coordinates {
          (0,{ln(1-0.066)-sqrt(0.066/(2000*(1-0.066)))})
          (10000,{ln(1-0.113)-sqrt(0.113/(2000*(1-0.113)))})
          (20000,{ln(1-0.194)-sqrt(0.194/(2000*(1-0.194)))})
          (30000,{ln(1-0.287)-sqrt(0.287/(2000*(1-0.287)))})
          (40000,{ln(1-0.381)-sqrt(0.381/(2000*(1-0.381)))})
          (50000,{ln(1-0.459)-sqrt(0.459/(2000*(1-0.459)))})
          (60000,{ln(1-0.497)-sqrt(0.497/(2000*(1-0.497)))})
          (70000,{ln(1-0.577)-sqrt(0.577/(2000*(1-0.577)))})
          (80000,{ln(1-0.612)-sqrt(0.612/(2000*(1-0.612)))})
          (90000,{ln(1-0.653)-sqrt(0.653/(2000*(1-0.653)))})
        };
        \addplot[fill=colUS, fill opacity=0.15] fill between[of=us4U and us4L];
        \addplot[name path=etc4, color=colETC, very thick, dashed, mark=triangle*] coordinates {
          (0,   {ln(1-0.048)}) (10000,{ln(1-0.075)}) (20000,{ln(1-0.085)})
          (30000,{ln(1-0.113)}) (40000,{ln(1-0.130)}) (50000,{ln(1-0.164)})
          (60000,{ln(1-0.185)}) (70000,{ln(1-0.204)}) (80000,{ln(1-0.243)})
          (90000,{ln(1-0.262)})
        };
        \addplot[name path=etc4U, draw=none] coordinates {
          (0,{ln(1-0.048)+sqrt(0.048/(2000*(1-0.048)))})
          (10000,{ln(1-0.075)+sqrt(0.075/(2000*(1-0.075)))})
          (20000,{ln(1-0.085)+sqrt(0.085/(2000*(1-0.085)))})
          (30000,{ln(1-0.113)+sqrt(0.113/(2000*(1-0.113)))})
          (40000,{ln(1-0.130)+sqrt(0.130/(2000*(1-0.130)))})
          (50000,{ln(1-0.164)+sqrt(0.164/(2000*(1-0.164)))})
          (60000,{ln(1-0.185)+sqrt(0.185/(2000*(1-0.185)))})
          (70000,{ln(1-0.204)+sqrt(0.204/(2000*(1-0.204)))})
          (80000,{ln(1-0.243)+sqrt(0.243/(2000*(1-0.243)))})
          (90000,{ln(1-0.262)+sqrt(0.262/(2000*(1-0.262)))})
        };
        \addplot[name path=etc4L, draw=none] coordinates {
          (0,{ln(1-0.048)-sqrt(0.048/(2000*(1-0.048)))})
          (10000,{ln(1-0.075)-sqrt(0.075/(2000*(1-0.075)))})
          (20000,{ln(1-0.085)-sqrt(0.085/(2000*(1-0.085)))})
          (30000,{ln(1-0.113)-sqrt(0.113/(2000*(1-0.113)))})
          (40000,{ln(1-0.130)-sqrt(0.130/(2000*(1-0.130)))})
          (50000,{ln(1-0.164)-sqrt(0.164/(2000*(1-0.164)))})
          (60000,{ln(1-0.185)-sqrt(0.185/(2000*(1-0.185)))})
          (70000,{ln(1-0.204)-sqrt(0.204/(2000*(1-0.204)))})
          (80000,{ln(1-0.243)-sqrt(0.243/(2000*(1-0.243)))})
          (90000,{ln(1-0.262)-sqrt(0.262/(2000*(1-0.262)))})
        };
        \addplot[fill=colETC, fill opacity=0.15] fill between[of=etc4U and etc4L];
        \addplot[name path=sr4, color=colSR, very thick, dashdotted, mark=square*] coordinates {
          (0,   {ln(1-0.000)}) (10000,{ln(1-0.420)}) (20000,{ln(1-0.609)})
          (30000,{ln(1-0.727)}) (40000,{ln(1-0.796)}) (50000,{ln(1-0.844)})
          (60000,{ln(1-0.892)}) (70000,{ln(1-0.923)}) (80000,{ln(1-0.949)})
          (90000,{ln(1-0.956)})
        };
        \addplot[name path=sr4U, draw=none] coordinates {
          (0,{ln(1-0.000)+0})
          (10000,{ln(1-0.420)+sqrt(0.420/(2000*(1-0.420)))})
          (20000,{ln(1-0.609)+sqrt(0.609/(2000*(1-0.609)))})
          (30000,{ln(1-0.727)+sqrt(0.727/(2000*(1-0.727)))})
          (40000,{ln(1-0.796)+sqrt(0.796/(2000*(1-0.796)))})
          (50000,{ln(1-0.844)+sqrt(0.844/(2000*(1-0.844)))})
          (60000,{ln(1-0.892)+sqrt(0.892/(2000*(1-0.892)))})
          (70000,{ln(1-0.923)+sqrt(0.923/(2000*(1-0.923)))})
          (80000,{ln(1-0.949)+sqrt(0.949/(2000*(1-0.949)))})
          (90000,{ln(1-0.956)+sqrt(0.956/(2000*(1-0.956)))})
        };
        \addplot[name path=sr4L, draw=none] coordinates {
          (0,{ln(1-0.000)-0})
          (10000,{ln(1-0.420)-sqrt(0.420/(2000*(1-0.420)))})
          (20000,{ln(1-0.609)-sqrt(0.609/(2000*(1-0.609)))})
          (30000,{ln(1-0.727)-sqrt(0.727/(2000*(1-0.727)))})
          (40000,{ln(1-0.796)-sqrt(0.796/(2000*(1-0.796)))})
          (50000,{ln(1-0.844)-sqrt(0.844/(2000*(1-0.844)))})
          (60000,{ln(1-0.892)-sqrt(0.892/(2000*(1-0.892)))})
          (70000,{ln(1-0.923)-sqrt(0.923/(2000*(1-0.923)))})
          (80000,{ln(1-0.949)-sqrt(0.949/(2000*(1-0.949)))})
          (90000,{ln(1-0.956)-sqrt(0.956/(2000*(1-0.956)))})
        };
        \addplot[fill=colSR, fill opacity=0.15] fill between[of=sr4U and sr4L];
        \addplot[name path=fcsr4, color=colFCSR, very thick, smooth, mark=o, mark options={solid}] coordinates {
          (0,   {ln(1-0.000)}) (10000,{ln(1-0.303)}) (20000,{ln(1-0.559)})
          (30000,{ln(1-0.712)}) (40000,{ln(1-0.849)}) (50000,{ln(1-0.906)})
          (60000,{ln(1-0.940)}) (70000,{ln(1-0.968)}) (80000,{ln(1-0.976)})
          (90000,{ln(1-0.990)})
        };
        \addplot[name path=fcsr4U, draw=none] coordinates {
          (0,{ln(1-0.000)+0})
          (10000,{ln(1-0.303)+sqrt(0.303/(2000*(1-0.303)))})
          (20000,{ln(1-0.559)+sqrt(0.559/(2000*(1-0.559)))})
          (30000,{ln(1-0.712)+sqrt(0.712/(2000*(1-0.712)))})
          (40000,{ln(1-0.849)+sqrt(0.849/(2000*(1-0.849)))})
          (50000,{ln(1-0.906)+sqrt(0.906/(2000*(1-0.906)))})
          (60000,{ln(1-0.940)+sqrt(0.940/(2000*(1-0.940)))})
          (70000,{ln(1-0.968)+sqrt(0.968/(2000*(1-0.968)))})
          (80000,{ln(1-0.976)+sqrt(0.976/(2000*(1-0.976)))})
          (90000,{ln(1-0.990)+sqrt(0.990/(2000*(1-0.990)))})
        };
        \addplot[name path=fcsr4L, draw=none] coordinates {
          (0,{ln(1-0.000)-0})
          (10000,{ln(1-0.303)-sqrt(0.303/(2000*(1-0.303)))})
          (20000,{ln(1-0.559)-sqrt(0.559/(2000*(1-0.559)))})
          (30000,{ln(1-0.712)-sqrt(0.712/(2000*(1-0.712)))})
          (40000,{ln(1-0.849)-sqrt(0.849/(2000*(1-0.849)))})
          (50000,{ln(1-0.906)-sqrt(0.906/(2000*(1-0.906)))})
          (60000,{ln(1-0.940)-sqrt(0.940/(2000*(1-0.940)))})
          (70000,{ln(1-0.968)-sqrt(0.968/(2000*(1-0.968)))})
          (80000,{ln(1-0.976)-sqrt(0.976/(2000*(1-0.976)))})
          (90000,{ln(1-0.990)-sqrt(0.990/(2000*(1-0.990)))})
        };
        \addplot[fill=colFCSR, fill opacity=0.15] fill between[of=fcsr4U and fcsr4L];
      \end{axis}
    \end{tikzpicture}
    \label{fig:synthetic_combined}
  \end{minipage}

\vspace{-27pt}
  \begin{center}
    \begin{tikzpicture}
        \matrix[matrix of nodes, column sep=10pt, row sep=0pt, anchor=center] {
        \draw[colUS, very thick, dashed] (0,0) -- (0.6,0);
        \fill[colUS] (0.3,0) node[draw=colUS, fill=colUS, diamond, inner sep=1.5pt] {}; &
        \node {US}; &
        \draw[colETC, very thick, dashed] (0,0) -- (0.6,0);
        \fill[colETC] (0.3,0) circle (0pt) node[draw=colETC, fill=colETC, regular polygon, regular polygon sides=3, inner sep=1pt] {}; &
        \node {ETC}; &
        \draw[colSR, very thick, dashdotted] (0,0) -- (0.6,0);
        \fill[colSR] (0.25,-0.08) rectangle ++(5pt,5pt) +(3pt,3pt); &
        \node {SR}; &
        \draw[colFCSR, very thick] (0,0) -- (0.6,0);
        \draw[colFCSR, very thick] (0.3,0) circle (2pt); &
        \node {FCSR}; \\
  };
\end{tikzpicture}
\end{center}
  \vspace{-18pt}
  
  \captionsetup{font=small}
  \caption{Performance comparison on synthetic instances. Each subplot shows
    $\ln(1-\mathrm{Accuracy})$ vs.\ budget for four algorithms (US, ETC, SR, FCSR).
    Shaded bands are $\pm1\sigma$ using the delta-method approximation
    $\mathrm{Var}(\ln(1-\hat{A})) \approx \tfrac{\hat{A}}{N(1-\hat{A})}$ with $N=2000$.}
  \label{fig:synthetic_all}
\end{figure*}

\subsection{Synthetic Instances}
We design four distinct problem instances to stress-test different aspects of the algorithms. These instances are parameterized by a small constant $a \in (0.001, 0.1)$ to control the difficulty. For notational convenience, define $[i:j] = \{i, i + 1, \dots, j\}$ where $i$ and $j$ are integers. Let $\mu_{[a:b], [c:d]} = x$ denote $\mu_{mk} = x$ for all $m \in [a:b]$ and all $k \in [c:d]$.

 \textbf{Experiment 1: Risky Instance.} We set $\tau = 0.5$ and $a = 0.01$. There is only feasible arm $i^\star = 10$ with $\mu_{i^\star, [1:M]} = 0.7$. The other $K-1=9$ arms are infeasible but possess high overall mean. $\mu_{[1:9], 5} = 0.5 - a \leq \tau$. $\mu_{[1:9], [1:4]} = 0.8 + \frac{a}{4}$. Thus, $\mu_{[1:9]} = 0.74$. 
    
    \emph{Results and Discussion.} Experiment $1$ is designed to challenge algorithms that pursue arms with a high overall mean without sufficiently verifying their feasibility. We observe that in Experiment 1, FCSR significantly outperforms the baseline algorithms that are less robust to risky arms, with SR in second place. See Figure \ref{fig:synthetic_all} for results of all 4 experiments.

 \textbf{Experiment 2: Feasibility Instance.}  We set $\tau = 0.5$ and $a=0.01$. The optimal arm $i^{\star} = 10$ has a high mean with $\mu_{i^\star, [1:4]}=0.8$, but its fifth attribute is barely feasible with $\mu_{i^\star, 5} =0.5+a$. The other $K-1=9$ arms are all comfortably feasible with $\mu_{[1:9],[1:5]} = 0.6$ but are significantly suboptimal. 
    
    \emph{Results and Discussion.} The difficulty in Experiment $2$ lies in gathering enough evidence to confirm the feasibility of the best arm without misidentifying one of the suboptimal but "safer" arms as the best. In Experiment 2, the Feasibility Instance and in Experiment 4, the Combined instance, FCSR still outperforms the baselines to a large extent, however the difference is not as pronounced as in Experiment 1.

 \textbf{Experiment 3: Mean Identification Instance.} We set $\tau = 0.3$ and $a=0.003$. All arms are clearly feasible. The best arm $i^\star = 1$ has $\mu_{i^\star, [1:5]} = 0.7$. The mean of each subsequent arm $k \in \{2,\ldots,10\}$ is progressively lower, decreasing as an arithmetic mean with parameter $a$, i.e. $\mu_k = \mu_{mk} = 0.7 - (k-1)a$ for all $k \in [2:10]$ and $m \in [1:5]$.
    
    \emph{Results and Discussion.} Experiment $3$ acts as a control and is equivalent to vanilla best arm identification. It tests FCSR on the performance degrade experienced due to partitioning its budget across various sampling strategies. In Experiment 3, we observe that SR is the best algorithm with FCSR not too far behind. This is to be expected as Experiment 3 is essentially a best arm identification problem.

 \textbf{Experiment 4: Combined Instance.} We set $\tau = 0.5$ and $a = 0.01$. This instance is a combination of the previous three instances. The best arm $i^\star = 10$ has arm mean $\mu_{i^\star} = 0.7$ with one attribute close to the threshold at $\mu_{i^\star, 5} = 0.5 + a$ while the rest are $\mu_{i^\star, [1:4]} = 0.75$. The first $5$ arms are risky with $\mu_{[1:5], 5} = 0.5 - a$, and $\mu_{[1:5], [1:4]} = 0.9 + \frac{a}{4}$. The next $4$ arms are difficult in the mean identification sense, i.e., all attributes are equal and the arm means decrease in an arithmetic progression with parameter $a$ from $\mu_{k} = 0.7$ to $\mu_{k} = 0.7 - 4a$ for $k \in [6:9]$. 

    \emph{Results and Discussion.} Experiment $4$ combines the above 3 difficult instances. We observe FCSR continues to observe a relatively steeper decay in log error probability and outperforms all other baselines. Since the problem instance consists of both mean discrimination and feasibility checking, we observe that the other algorithms, with SR in specific, is not as disfavored in this experiment as compared to Experiment $1$.

\begin{table*}[!t]
\centering
\footnotesize                       
\setlength{\tabcolsep}{8pt}         
\renewcommand{\arraystretch}{1}  
\caption{A movie portfolio of films of various genres. Rows with true mean $<0.73$ are shaded.}
\label{tab:movie-ratings-colored}
\begin{tabularx}{\textwidth}{@{} c c c l >{\raggedright\arraybackslash}X c @{}}
\toprule
Arm No. & True Arm mean & Attr. No. & Genre & Movie title & True attribute mean \\
\midrule
\multirow{5}{*}{0} & \multirow{5}{*}{0.826} &
1 & Comedy   & Princess Bride, The (1987)   & 0.826 \\
                   &                     &
2 & Action   & Star Wars: Episode IV - A New Hope (1977)               & 0.824 \\
                   &                     &
3 & Drama    & American Beauty (1999)                                  & 0.821 \\
                   &                     &
4 & Thriller & Dark City (1998)                                       & 0.761 \\
                   &                     &
5 & Sci\-Fi  & Army of Darkness (1993)                                 & 0.747 \\
\midrule
\multirow{5}{*}{1} & \multirow{5}{*}{0.740} &
1 & Comedy   & Blazing Saddles (1974)                                  & 0.772 \\
                   &                     &
2 & Action   & Star Wars: Episode VI - Return of the Jedi (1983)      & 0.799 \\
                   &                     &
3 & Drama    & Bridge on the River Kwai, The (1957)                   & 0.819 \\
\rowcolor{red!15}  &                     &
4 & Thriller & Con Air (1997)                                         & 0.640 \\
\rowcolor{red!15}  &                     &
5 & Sci\-Fi  & X-Files: Fight the Future, The (1998)                  & 0.668 \\
\midrule
\rowcolor{red!15} \multirow{5}{*}{2} & \multirow{5}{*}{0.710} &
1 & Comedy   & My Cousin Vinny (1992)                                  & 0.721 \\
\rowcolor{red!15}  &                     &
2 & Action   & Mission: Impossible (1996)                              & 0.680 \\
                   &                     &
3 & Drama    & Leaving Las Vegas (1995)                                & 0.735 \\
\rowcolor{red!15}  &                     &
4 & Thriller & Devil's Advocate, The (1997)                            & 0.709 \\
\rowcolor{red!15}                   &                     &
5 & Sci\-Fi  & Mad Max (1979)                                          & 0.706 \\
\bottomrule
\end{tabularx}
\end{table*}

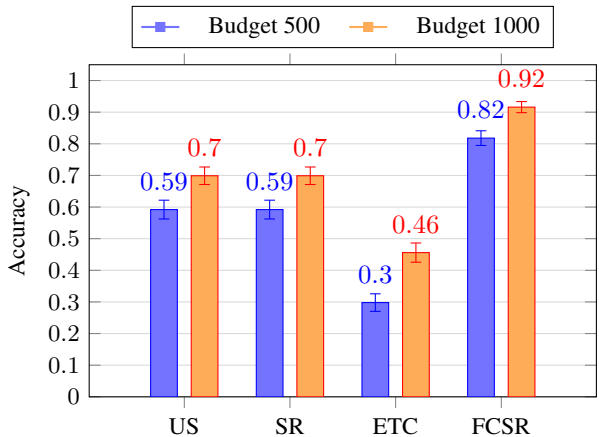
\begin{figure}[t]
\centering
\begin{tikzpicture}
  \newcommand{\Ntrials}{1000} 

  \pgfmathsetmacro{\errUSfive}{1.96*sqrt(0.592*(1-0.592)/\Ntrials)}
  \pgfmathsetmacro{\errUSthousand}{1.96*sqrt(0.699*(1-0.699)/\Ntrials)}

  \pgfmathsetmacro{\errSRfive}{1.96*sqrt(0.592*(1-0.592)/\Ntrials)}
  \pgfmathsetmacro{\errSRthousand}{1.96*sqrt(0.699*(1-0.699)/\Ntrials)}

  \pgfmathsetmacro{\errETCfive}{1.96*sqrt(0.298*(1-0.298)/\Ntrials)}
  \pgfmathsetmacro{\errETCthousand}{1.96*sqrt(0.456*(1-0.456)/\Ntrials)}

  \pgfmathsetmacro{\errFCSRfive}{1.96*sqrt(0.818*(1-0.818)/\Ntrials)}
  \pgfmathsetmacro{\errFCSRthousand}{1.96*sqrt(0.916*(1-0.916)/\Ntrials)}

  \begin{axis}[
      ybar,
      bar width=10pt,
      width=\linewidth,
      height=6cm,
      ymin=0,
      ymax=1.05,
      ylabel={Accuracy},
      symbolic x coords={US,SR,ETC,FCSR},
      xtick=data,
      enlarge x limits=0.3,
      tick label style={font=\small},
      label style={font=\small},
      ytick style={font=\small},
      ytick={0,0.1,...,1.0},
      ymajorgrids=true,
      grid style={draw=gray!30},
      legend style={
        font=\small,
        at={(0.5,1.05)},
        anchor=south,
        legend columns=2,
        column sep=0.3cm,
        row sep=2pt},
      nodes near coords,
      nodes near coords align={vertical},
      every node near coord/.append style={yshift=4pt}
  ]

  \addplot+[
    ybar,
    bar shift=-7.6pt,
    fill=blue!55,
    legend image code/.code={
      \draw[blue!55, thick] (-0.25cm,0cm) -- (0.25cm,0cm);
      \fill[blue!55] (-0.06cm,-0.06cm) rectangle (0.06cm,0.06cm);
    },
    error bars/.cd,
      y dir=both, y explicit
  ]
  coordinates {
    (US,0.592)  +- (0,\errUSfive)
    (SR,0.592)  +- (0,\errSRfive)
    (ETC,0.298) +- (0,\errETCfive)
    (FCSR,0.818)+- (0,\errFCSRfive)
  };

  \addplot+[
    ybar,
    bar shift=+7.6pt,
    fill=orange!65,
    legend image code/.code={
      \draw[orange!65, thick] (-0.25cm,0cm) -- (0.25cm,0cm);
      \fill[orange!65] (-0.06cm,-0.06cm) rectangle (0.06cm,0.06cm);
    },
    error bars/.cd,
      y dir=both, y explicit
  ]
  coordinates {
    (US,0.699)  +- (0,\errUSthousand)
    (SR,0.699)  +- (0,\errSRthousand)
    (ETC,0.456) +- (0,\errETCthousand)
    (FCSR,0.916)+- (0,\errFCSRthousand)
  };
  \legend{Budget 500 , Budget 1000}
  \end{axis}
\end{tikzpicture}
\caption{Accuracy of different algorithms at two budgets. Error bars show 95\% confidence intervals for a Bernoulli mean using the empirical variance $\widehat{\mathrm{Var}}(X)=\hat p(1-\hat p)$, i.e., $\hat p \pm 1.96\sqrt{\hat p(1-\hat p)/N}$ with $N=1000$ independent runs.}
\label{fig:bar-accuracy}
\end{figure}

\subsection{MovieLens Dataset}
We construct a grouped bandit instance using the MovieLens-25M dataset \citep{10.1145/2827872}. In this setup, each arm represents a "movie portfolio" composed of films from different genres (attributes). The goal is to identify the portfolio where every genre has high audience appeal, mirroring a content provider's need to curate a bundle that is consistently satisfying across different demographics. 

We define an arm as feasible if the average normalized rating for each of its constituent movies exceeds a threshold (we set the threshold $\tau = 0.73$, equivalent to 3.65/5), ensuring a baseline quality for all genres in the portfolio. The best arm is the feasible portfolio with the highest overall average rating. We filter the movie set to include only films with at least 800 user ratings to ensure rating stability. From the remaining movies, we identify the top 5 most frequent genres to serve as our attributes $M=5$, and we define $K=3$ arms, where each arm represents a portfolio of movies.

The true mean reward for attribute $j$ of arm $i$ is defined as the average of all user ratings for the corresponding movie, normalized to the $[0,1]$ interval by dividing by 5.0. When the bandit algorithm pulls this attribute, it receives stochastic rewards sampled uniformly with replacement from the actual historical ratings of that movie.

\paragraph{Discussion.} We observe the FCSR  outperforms baselines on the MovieLens dataset for the instance in Table \ref{tab:movie-ratings-colored}, for both $T=500$ and $T=1000$ (plotted in that order respectively in Figure \ref{fig:bar-accuracy}). Thus, FCSR is viable even for practical applications and small budget regimes. 


\section{CONCLUSIONS}
We introduced Feasibility Constrained Successive Rejects (FCSR), a parameter free algorithm for fixed-budget best arm identification in grouped bandits. We define a complexity parameter $H_{FC}$ and use it to derive a lower bound for the setting. We prove an upper bound on the error probability of FCSR and show that it matches the lower bound up to constants, demonstrating the optimality of FCSR. To this end, we propose $\textsc{SampleUntilFeasible}$, a novel sub-routine necessary for optimal performance. Empirically, FCSR outperforms natural baselines on several synthetic stress–tests and is viable for practical applications in low budget regimes.

\bibliography{refs} 

\newpage

\onecolumn

\title{Fixed-Budget Constrained Best Arm Identification in Grouped Bandits\\(Supplementary Material)}
\maketitle
\appendix
\section{PSEUDOCODE}
Let $X_{i,j}(t) \sim \nu_{i,j}$ denote the random reward observed on pulling attribute $(i,j), i \in [K], j\in [M]$ at time $t > 1$. Let $C_{i,j}(t)$ denote the number of times $(i,j)$ was pulled before time $t$, $S_{i,j}$, the sum of all observed rewards before $t$, and $\hat{\mu}_{i,j}(t) = S_{i,m} (t) / \max \{C_{i,m} (t), 1\}$, the attribute empirical mean. Thus, on pulling attribute $(i,j)$ at time $t$ we update 
\begin{equation}
\label{eq:upd-rule}
S_{i,j}(t+1) \gets S_{i,j}(t) + X_{i,j}(t)\,, \quad C_{i,j}(t+1) \gets C_{i,j}(t) + 1, \quad \hat{\mu}_{i,j}(t+1) = \frac{S_{i,m} (t+1)}{\max \{C_{i,m} (t+1), 1\}}.
\end{equation}

Let $S(t) \in \mathbb{R}^{K \times M}$ be a matrix where each entry is given by $S_{i,j}(t)$ and $S_{(i)}(t)$ be row $i$ of $S(t)$ (i.e. the vector of sums of arm $i$). Similarly, define $C(t) \in \mathbb{N}^{K \times M}$, $C_{(i)}(t) \in \mathbb{N}^{M} \, \forall \, i$. The scoring rule is defined for each arm $i$ at some time $t$ as
\begin{equation}
    \label{eq:score-supm}
    s(i;t) :=
    \begin{cases}
        \hat\mu_i(t) &\text{if }\min_{j\in[M]}\hat\mu_{ij}(t)>\tau,\\
        \min_{j\in[M]}\hat\mu_{ij}(t)  &\text{otherwise.}
    \end{cases}
\end{equation}
For the sake of simplicity, we omit the time step $t$ when referring to the global attribute statistics $S, C, \hat{\mu}$ and the reward $X_{i,j}$ used by \textsc{FCSR}. However, we explicitly index any statistic that is used only within the scope of a sub-routine. Further, recall the definitions of the per-round budget of an arm $\Delta n_r = n_r - n_{r-1}$ (with $n_0 = 0)$ and for all $r > 0$, $n_r$ is defined as  $n_r \gets \lceil \lfloor (1-f) T \rfloor \, /\, \bar{n} (K + 1 - r) \rceil$. The explicit form of the normalizing constant is $\bar{n} \gets \frac{1}{2} + \sum_{k=2}^{K} \frac{1}{k}$.

\begin{algorithm}[H]
\caption{$\textsc{FCSR}(T, \tau, f, g \;;\; S, C)$}
\label{alg:fcsr}

\begin{algorithmic}[1]
\Require $K$ arms, $M$ attributes/arm, budget $T$, threshold $\tau$, fractions $f, g \in (0,1)$ and matrices $S, C$.
\Ensure The best feasible arm or -1.

\State \textbf{Initialize:}  Global statistics $S_{i,j}, C_{i,j}, \hat{\mu}_{i,j} \gets 0 \;\forall\; i \in [K], j\in[M].$
\State $P_i \gets \lfloor \frac{fT}{K} \rfloor \; \forall \; i$. $T_{\text{extra}} \gets 0.$
\For{$r = 1$ to $K-1$}
    \State $n_r \gets \lceil \lfloor (1-f) T \rfloor \, /\, \bar{n} (K + 1 - r) \rceil$, and  $\Delta n_r \gets n_r - n_{r-1}$ (where $n_0 = 0$).    
    \State $T_{\text{extra}}^{\text{arm}} \gets T_{\text{extra}} / K$; $T_{\text{extra}} \gets 0$.

    \For{each arm $i \in \mathcal{A}$} \Comment{Sample each active arm}
        \State Run $\textsc{Uniform}(i, \lfloor (1-g)  \cdot \Delta n_r \rfloor + T_{\text{extra}}^{\text{arm}} \;;\; S_{i}, C_{i})$.
        \State Run $\textsc{APT}(i, \lfloor g \cdot \Delta n_r \rfloor, \tau\;;\; S_i, C_i)$.
        \State Run $\textsc{SampleUntilFeasible}(i, P_i, \tau \;;\; S_i, C_i)$.
    \EndFor
    
    \State Compute $s(i) \; \forall \; i$ using Equation \ref{eq:score-supm}, and eliminate $\mathcal{A} \gets \mathcal{A} \setminus \{ \arg \min_{i\in \mathcal{A}} s(i) \}$.
    \State $T_{\text{extra}} \gets T_{\text{extra}} + P_{\arg \min_{i\in \mathcal{A}}s(i)}$. \Comment{Repurpose feasibility budget of eliminated arm}
\EndFor

\State \textbf{If }{the last remaining arm $i_{\text{final}} \in \mathcal{A}$ is feasible} \textbf{then }\textbf{return} $i_{\text{final}.}$
\State \textbf{Else return} $0$.

\end{algorithmic}
\end{algorithm}

\begin{algorithm}
\caption{$\textsc{Uniform}(i, T; S_i, C_i)$}
\label{alg:us}

\begin{algorithmic}[1]
\Require Arm $i$, Budget $T$ and global statistics $S_i, C_i$
\For{attribute $j = 1, 2, \dots M$ of arm $i$}
    \For{samples $s = 1, 2, \dots \lfloor T/M\rfloor$} \Comment{Sample each attribute uniformly}
        \State Observe reward $X_{i,j} \sim \nu_{i,j}$
        \State Update global $S_{i, J_{t+1}} \,, \, C_{i,J_{t+1}}$ and $\hat{\mu}_{i, J_{t+1}}$ using Equation \ref{eq:upd-rule}.
    \EndFor
\EndFor
\State \Return Updated statistics $C_i, S_i$
\end{algorithmic}
\end{algorithm}

\begin{algorithm}
\caption{$\textsc{APT}(i, T; S_i, C_i)$}
\label{alg:apt}
\begin{algorithmic}[1]
\Require Arm $i$, Budget $T$, Threshold \(\tau \in \mathbb{R}\), and global statistics $S_i, C_i$
\State \textbf{Initialization:} Use existing global $S_i, C_i, \hat{\mu}_{i,j} \;\forall\; j \in [M]$. 
\For{\(t = 0, 1, \dots, T-1\)}
    \For{each attribute \(j = 1, \dots, M\)}
        \State Compute the empirical gap:
        $\widehat{\bar{\Delta}}_{j}(t) = |\hat{\mu}_{i,j} - \tau| $
        \State Compute the APT score: $B_j(t+1) = \sqrt{C_{i,j}} \cdot \widehat{\bar{\Delta}}_{ij}(t)$
    \EndFor
    \State Choose attribute \(J_{t+1} = \arg\min_{1 \leq j \leq M} B_j(t+1)\) and pull it.
    \State Observe reward \(X_{i,J_{t+1}} \sim \nu_{i, J_{t+1}}\).
    \State Update global $S_{i, J_{t+1}} \,, \, C_{i,J_{t+1}}$ and $\hat{\mu}_{i, J_{t+1}}$ using Equation \ref{eq:upd-rule}.
\EndFor

\State \Return Updated statistics $S_i, C_i$.

\end{algorithmic}
\end{algorithm}

\begin{algorithm}[H]
\caption{\textsc{SampleUntilFeasible}$(i, P_i; S_{i}, C_{i}, \tau)$}
\label{alg:suf}
\begin{algorithmic}[1]
\Require Arm $i$, feasibility budget $P_i$, global $S_i \in \mathbb{R}^{M}$, $C_i \in \mathbb{N}^{M}$, threshold $\tau$
\State \textbf{Initialization:} Use existing $S_i, C_i, \hat{\mu}_{i,j} \;\forall\; j \in [M].$
\While{$P_i \geq 1$} 
\State $\mathcal{F}^{c, (i)} \gets \{j : \hat{\mu}_{i,j} \leq \tau\}$ \Comment{Set of infeasible attributes} 
\State Break if $\mathcal{F}^{c, (i)} = \emptyset$
\State $j^* \gets \min_{j \in \mathcal{F}^{c, (i)}}j$ \Comment{Select lowest index infeasible attribute} 
\While{$\hat{\mu}_{i,j^\star} \leq \tau$} \Comment{Sample until empirical mean above threshold}
\State Observe reward $X_{i,j^\star} \sim \nu_{i,j^\star}$
\State Update global $S_{i, j^\star} \,, \, C_{i,j^\star}$ and $\hat{\mu}_{i, j^\star}$ using Equation \ref{eq:upd-rule}.

\State $P_i \gets P_i - 1$
\State Break if $P_i = 0$ \Comment{ Break when feasibility budget exhausted} 

\EndWhile
\EndWhile
\State \Return Updated statistics $S_i, C_i$.
\end{algorithmic}
\end{algorithm}

The following lemma demonstrates the validity of FCSR in the sense that it does not exceed the given budget $T$.

\begin{lemma}[Budget compliance]
Let $T\in\mathbb N$ and $f,g\in(0,1)$. In \textsc{FCSR}, set $P_i=\lfloor fT/K\rfloor$ and $T_{\mathrm{SR}}=T-\sum_{i=1}^K P_i$. Let $n_0=0<n_1<\cdots<n_{K-1}$ be defined by
\[
n_r=\Big\lceil \frac{\lfloor(1-f)T\rfloor}{\bar n\,(K+1-r)}\Big\rceil,\qquad \Delta n_r=n_r-n_{r-1}.
\]
In round $r$, each active arm uses $\lfloor g\Delta n_r\rfloor$ APT pulls and $\lfloor(1-g)\Delta n_r\rfloor$ uniform pulls. Then the total number of pulls made by the algorithm is at most $T$.
\end{lemma}

\begin{proof}
Split pulls into feasibility and SR pulls. Feasibility: each arm can consume at most $B_i$ pulls from its feasibility budget; any reallocated feasibility pulls come from the same pool. Hence
\[
P_{\mathrm{feas}}\le \sum_{i=1}^K P_i \le fT 
\]

SR: in round $r$ there are $K+1-r$ active arms. For each such arm,
\[
\lfloor g\Delta n_r\rfloor+\lfloor(1-g)\Delta n_r\rfloor \le \Delta n_r,
\]
so the SR pulls in round $r$ are at most $(K+1-r)\Delta n_r$. Summing over rounds gives
\[
P_{\mathrm{SR}} \le \sum_{r=1}^{K-1} (K+1-r)\Delta n_r.
\]
By construction of $n_r$ and monotonicity of floor/ceiling, replacing the ideal real-valued schedule
$\tilde n_r=\frac{\lfloor(1-f)T\rfloor}{\bar n\,(K+1-r)}$ by $n_r=\lceil \tilde n_r\rceil$ can only decrease the weighted sum $\sum_{r}(K+1-r)\Delta n_r$ by at most rounding effects, and therefore
\[
\sum_{r=1}^{K-1} (K+1-r)\Delta n_r \le \lfloor(1-f)T\rfloor \le T_{\mathrm{SR}} .
\]
Hence $P_{\mathrm{SR}}\le T_{\mathrm{SR}}$. Total: combining,
\[
P_{\mathrm{total}} = P_{\mathrm{feas}}+P_{\mathrm{SR}} \le \sum_{i=1}^K P_i + T_{\mathrm{SR}} = T.
\]
\end{proof}

\section{Discussion on Terminology}
\label{app:just}
An algorithm in the fixed budget literature is generally dubbed \emph{optimal} if the terms in the exponent on the upper bound match the terms in the lower bound up to constant factors. See \citet{carpentier2016tightlowerboundsfixed} for the usage of "optimal" to describe APT. The lower bounds are obtained by considering a class of bandit instances and lower bounding the worst case error of any algorithm on this class.  

\citet{carpentier2016tightlowerboundsfixed} also dub Successive Rejects as \emph{optimal} since it matches the lower bound proved by them up to constant factors in the exponent. The lower bound in this case also considers an adversarial class of bandit instances with bounded difficulty as in this paper. Even in the vanilla single dimensional unconstrained best arm identification problem, no algorithm is known to match a lower bound exactly \citep{qin2023openproblemoptimalbest}. Since we employ similar methods for the lower bound and FCSR attains them up to constants in the exponent, FCSR is an optimal algorithm for this setting.

\section{PROOF OF LOWER BOUND}
\label{app:lb}

\subsection{Proof Strategy}
 We define two classes of bandit instances, $\mathcal{C}_F$ the "Feasibility Class" and $\mathcal{C}_{R}$ the "Risky Class" to lower bound the probability of error some arbitrary policy $Alg$ makes due to (i) deeming the best feasible arm to be infeasible and (ii) misreporting either a feasible sub-optimal arm or an infeasible arm with high mean to be the best arm. We then argue that the lower bound on the error probability of any arbitrary algorithm $Alg$, is at least the maximum of the worst case error probability on these individual classes, i.e.
 $$
 \max_{\mathcal{G} \in \mathcal{C}_F \cup \mathcal{C}_R} \mathbb{P}_{\mathcal{G}}(e) \geq \max\left\{
 \max_{\mathcal{G} \in \mathcal{C}_F }\mathbb{P}_{\mathcal{G}} (e), 
 \max_{\mathcal{G} \in \mathcal{C}_R}\mathbb{P}_{\mathcal{G} }(e)
 \right\}
 $$
where the error event is defined as $e = 1\{I_T \neq i^\star\}$ Although we assume $i^\star$ exists in our constructions and make further assumptions on $K$ and $M$, we note that the lower bound nevertheless applies to any $Alg$ since we are able to identify a particular problem instance where $Alg$ does not perform well.

We use Lemma \ref{lem:tbpg} to establish the lower bound on the Feasibility Class and we provide a novel multi-dimensional construction based on \cite{carpentier2016tightlowerboundsfixed} and \cite{locatelli2016optimalalgorithmthresholdingbandit} to capture the tradeoffs $Alg$ must make on the Risky Class.

\subsection{Feasibility Class}

We first provide an lower bound for Thresholding Bandit Problem (TBP)  \citep{locatelli2016optimalalgorithmthresholdingbandit} for grouped arms. The main difference is the proof in \cite{locatelli2016optimalalgorithmthresholdingbandit} utilizes a single dimensional bandit instance where the arm distributions are Gaussian, whereas for the purposes of our overall proof we require a grouped bandit instance where arm distributions are be Bernoulli. Our proof recovers the result in \cite{locatelli2016optimalalgorithmthresholdingbandit} if $M=1$  with a condition on the range of the budget $T$. To this end, we adapt the proof of the tight lower bound proof for the vanilla best arm identification bandit problem found in \cite{carpentier2016tightlowerboundsfixed}.

\subsubsection{Multidimensional TBP.} 
The setting is very similar to the grouped bandit problem. We are given $K$ arms with $M$ attributes each and at the beginning of every round $t>0$, the learner choose to sample an attribute $(k_t,m_t)$ where $k\in[K], m\in[M]$ and observes a random reward $X(t) \sim \nu_{k,m}$. After $T$ rounds, the learner must output a set of arms $\widehat{S}_{\tau} \subseteq [K]$. Let $\widehat{S}_{\tau}^C = [K] /\widehat{S}_{\tau}$. As before, let $\mu_{k,m}$ denote the attribute mean of $(k,m)$ and $\mu_k$ the overall arm mean of arm $k$.

For some $\tau \in \mathbb{R}$, define the set of arms that are \emph{feasible} as $\mathcal{S}_{\tau} := \{k: \mu_k > \tau\}$. The set of \emph{infeasible} arms is denoted as $S^C_{\tau} := [K] /S_{\tau}$. The event that a learner misclassifies the feasibility of any choice is
$$
    \mathcal{L}(T) = 1\{\widehat{S}_{\tau} \cap S^C_{\tau} \neq \emptyset \; \lor \; \widehat{S}_{\tau}^C \cap S_{\tau} \neq \emptyset\}.
$$

The learner aims to minimize $\mathbb{E}[ \mathcal{L}(T)]$. This is the Thresholding Bandit Problem for the special case that $\epsilon = 0$ and arms with mean $\mu_k = \tau$ are classified as infeasible.

Define the \emph{feasibility class} $\mathcal{C}_F(d;K) := \{\mathcal{G}^k :k\in\{0\} \cup[K]\}$, a set of $K+1$ bandit problems that are parametrized by the gap constant $d \in (0, \frac{1}{4}]$ and $K=M$, the number of arms (and attributes) . Let $\tau = 0.5$ for all these problem instances and $d \in (0, \frac{1}{4}]$ be a constant that parametrizes the difficulty of each instance. Let $\nu' = \mathcal{B}(\frac{1}{2} - d)$ and $\nu= \mathcal{B}(\frac{1}{2} + d)$ where $\mathcal{B}(p)$ denotes a Bernoulli distribution with mean parameter $p$. For $0 \le k \le K$, $\mathcal{G}^k$ is problem instance where the distribution of arm $k \in [K]$ is given as $\nu_{k,m} = \nu \;\forall\; m\in[M]$. And for all arms $j \neq k$, the arm distribution is $\nu_{j,m} = \nu' \;\forall\; m\in[M]$.

\begin{lemma}[Bernoulli TBP Lower Bound]
    \label{lem:tbpg}
    For any $K= M \geq 2$ $d \in (0, \frac{1}{4}]$, if $T \geq 4(\log(K)H_f)^2\log(6TK)/(60)^2$, it holds for any bandit algorithm that
\begin{equation}
    \max_{\mathcal{G}^i \in \mathcal{C}_F(d;K)} 
        \mathbb{E}_{\mathcal{G}^i} [\mathcal{L}(T)] \ge 
        \frac{1}{6}\exp\left(\frac{-120T}{\log(K)H_f}\right)
\end{equation}
where $\mathbb{E}_{\mathcal{G}^i}$ is the expectation according to the samples of problem $\mathcal{G}^i$.
\end{lemma}

\begin{proof}
We use the strategy of the proof from \cite{carpentier2016tightlowerboundsfixed} with minor adaptations. Consider the event where (under some suitable conditions) some arbitrary learner $Alg$ reports that there are no infeasible choices/attributes, i.e. $\widehat{\mathcal{S}}^C_{\tau} = \emptyset$ in bandit instance $\mathcal{G}^0$. We then, via a change of measure argument, lower bound the probability that $Alg$ erroneously does not change its output even in instance $\mathcal{G}^i$ where the $i$-th choice is infeasible by construction. 

Since for $\mathcal{G}^k$, arm $k$ is feasible and all other arms are infeasible, $k= i^\star$. Thus, $\Delta_{i^\star j}^{-2} = d^{-2} \;\forall\; j \in [M]$. We note that 
\begin{equation}
\label{eq:lbb-f}
Md ^{-2} = \log K \frac{K}{\log K}\max_j\bar\Delta_{i^\star j}^{-2} = \log(K)H_f.
\end{equation}

\paragraph{Step 1: High probability event where empirical KL divergences concentrate} For two distributions $\nu,\nu'$ defined on $\mathbb R$ and that are such that $\nu$ is absolutely continuous with respect to $\nu'$, we write
$\text{KL}(\nu,\nu') = \int_{\mathbb R} \log\Big(\frac{d\nu(x)}{d\nu'(x)}\Big)d\nu(x),$
for the Kullback leibler divergence between distribution $\nu$ and $\nu'$. Let $k \in \{1, ..., K\}$. Let us write
$$ \text{KL} := \text{KL}(\nu', \nu) = \text{KL}(\nu, \nu') = (1 - 2p)\log\big(\frac{1 - p}{p}\big),$$
for the Kullback-Leibler divergence between two Bernoulli distributions $\nu$ and $\nu'$ of parameter $p = 0.5 + d$ and $1-p=0.5-d$. Since $p \in [1/4,1/2)$, the following inequality holds:
\begin{equation}\label{eq:gapKL}
\text{KL} \leq 10d^2.
\end{equation}
Let $1\leq t\leq T$. We define the empirical arm KL divergence as:
\begin{align*}
\widehat{\text{KL}}_{k}(t)
&= \frac{1}{t} \sum_{s=1}^t \sum_{m=1}^{M}\mathbf 1\{X_{k,m}(s) = 1\} \log(\frac{p}{1-p}) + \mathbf 1\{X_{k,m}(s) = 0\}\log(\frac{1-p}{p}),
\end{align*}
where by definition for any $s \leq t$, $X_{k,m}(s)\sim_{i.i.d} \nu_{k,m}^i$. Define the event that the arm KL divergences concentrate as
\begin{align*}
\xi &= \Big\{\forall 1 \leq k\leq K, \forall 1 \leq t\leq T, |\widehat{\text{KL}}_{k}(t)| - \text{KL}_k\leq 2 \ \sqrt[]{\frac{\log(6TK)}{t}} \Big\}.
\end{align*}
We now state the following claim; a concentration bound for $|\widehat{\text{KL}}_{k,t}|$ that holds for all bandit problems $\mathcal{G}^i$ with $0\leq i\leq K$.
\paragraph{Claim.}\label{xi}
It holds that
$\mathbb P_{\mathcal{G}^i}(\xi) \geq 5/6.$

We verify the above claim. If $k \neq i$ (and thus $\nu_{k,m}^i = \nu \;\forall\;m$) then $\mathbb E_{\mathcal G^k} [\widehat{\text{KL}}_{k}(t)] = \text{KL}$. If $k = i$ (and thus $\nu_{k,m}^i = \nu'$) then $\mathbb E_{\mathcal G^i} [\widehat{\text{KL}}_{k}(t)]  = -\text{KL}$ (since arm parameters are swapped for $k=i$). Moreover note that since $p \in [1/4,1/2)$
$$|\log(\frac{d \nu_{k,m}}{d \nu_{k,m}'}(X_{k,m}(s)))| = |\mathbf 1\{X_{k,m}(s) = 1\} \log(\frac{p}{1-p}) + \mathbf 1\{X_{k,m}(s) = 0\}\log(\frac{1-p}{p})| \leq \log(3).$$
Therefore, $\widehat{\text{KL}}_{k}(t)$ is a sum of i.i.d.~samples that are bounded by $\log(3)$, and whose mean is $\pm \mathrm{KL}$ depending on the value of $i$. We can apply Hoeffding's inequality to this quantity and we have that with probability larger than $1-(6KT)^{-1}$
$$|\widehat{\text{KL}}_{k}(t)| - \text{KL}_k\leq \sqrt{2}\log(3) \ \sqrt[]{\frac{\log(6TK)}{t}}.$$
This assertion and an union bound over all $1 \leq k \leq K$ and $1 \leq t \leq T$ implies that $\mathbb P_{\mathcal G^i}(\xi) \geq 5/6$, as we have $\sqrt{2}\log(3) < 2$, and thus the claim is verified.

\paragraph{Step 2: A change of measure} 
Let some arbitrary policy $\mathcal{A}lg$ return $\widehat{S}_{\tau}$ at the end of the budget $T$. Let $T_{k,m}$ denote the numbers of samples collected by $\mathcal{A}lg$ on each choice $(k,m)$ and $T_k = \sum_{m}T_{k,m}$ be the number of samples collected for arm $k$. These quantities are stochastic but it holds that $\sum_{k,m} T_{k,m} = \sum_k T_k =T$ by definition of the fixed budget setting. Let us write for any  $t_{k,m} = \mathbb E_{\mathcal{G}^0} [T_{k,m}].$ and $t_{k} = \mathbb E_{\mathcal{G}^0} [T_{k}].$ It holds also that $\sum_{k,m} t_{k,m} =\sum_{k} t_{k} = T$.

We recall the change of measure identity (see e.g. \cite{Audibert2010BestArmIdentification}) which states that for any measurable event $\mathcal{E}$ and for any $1 \leq i \leq K$ :
\begin{equation}
\label{cm}
\mathbb{P}_{\mathcal{G}^i}(\mathcal{E}) = \mathbb{E}_{\mathcal{G}^0} \Big[\mathbf{1}\{\mathcal{E}\}\exp \big(-T_i \widehat{\text{KL}}_{i,T_{i}}\big)\Big],
\end{equation}
as the product distributions $\mathcal G^i$ and $\mathcal G^0$ differ in all attributes $j\in[M]$ in arm $i$ and $T_i \widehat{\text{KL}}_{i,T_{i}}$ serves as the log-likelihood ration of all samples observed under arm $i$ in these two different bandit problems. Let $1\leq i \leq K$. Consider now the event 
$$\mathcal{E}_i = \{\widehat{S}_{\tau} = \emptyset\} \cap \{\xi\} \cap \{T_i \leq 6 t_i\},$$
i.e.~the event where the algorithm reports all choices as feasible at the end, where $\xi$ holds, and where the number of times choice $i$ was pulled is smaller than $6t_i$. We have by Equation~\eqref{cm}
\begin{align}
\mathbb{P}_{\mathcal{G}^i}(\mathcal{E}_i) & =\mathbb{E}_{\mathcal{G}^0}\Big[\mathbf{1}\{\mathcal{E}_i\}\exp\big(-T_i\widehat{\text{KL}}_{i,T_i}\big)\Big]\nonumber\\
& \geq \mathbb{E}_{\mathcal{G}^0}\Big[\mathbf{1}\{\mathcal E_i\}\exp\Big(-T_i\text{KL} -2 \ \sqrt[]{T_i\log(6TK)}\Big)\Big] \nonumber\\
& \geq \mathbb{E}_{\mathcal{G}^0}\Big[\mathbf{1}\{\mathcal E_i\}\exp\Big(-6t_i\text{KL} -2 \ \sqrt[]{T\log(6TK)}\Big)\Big]\nonumber\\
& \geq \exp\Big(-6t_i\text{KL} -2 \ \sqrt[]{T\log(6TK)}\Big) \mathbb{P}_{\mathcal{G}^0}(\mathcal E_i),\label{eq:event}
\end{align}
since on $\mathcal E_i$, we have that $\xi$ holds and that $T_i \leq 6t_i$, and since $\mathbb E_{1} \widehat{\text{KL}}_{i,t} = \text{KL}_i$ for any $t \leq T$.

\paragraph{Step 3 : Lower bound on $\mathbb{P}_{\mathcal{G}^0}(\mathcal E_i)$ for any reasonable algorithm} Assume that for the algorithm $\mathcal Alg$ that we consider
\begin{equation}\label{eq:bras0}
\mathbb E_{\mathcal{G}^0}[\widehat S_{\tau} = \emptyset] \leq 1/2,
\end{equation}
i.e.~that the probability that $\mathcal Alg$ makes a mistake on problem $\mathcal{G}^0$ is less than $1/2$. Note that if $\mathcal Alg$ does not satisfy that, it performs badly on problem $\mathcal{G}^0$ and its probability of success is not larger than $1/2$ uniformly on the $K+1$ bandit problems we defined.

For any $1 \leq k \leq K$ it holds by Markov's inequality and $\mathbb E_{\mathcal{G}^0} [T_k] = t_k$ that,
\begin{align}\label{eq:ma}
\mathbb P_{\mathcal{G}^0} (T_k \geq 6 t_k) \leq \frac{\mathbb E_{\mathcal{G}^0} [T_k]}{6t_k} = 1/6.
\end{align}
So by combining Equations~\eqref{eq:bras0},~\eqref{eq:ma} and Claim~\ref{xi}, it holds by an union bound that for any $2 \leq i \leq K$
$$\mathbb P_{\mathcal{G}^0}(\mathcal E_i) \geq 1 - (1/6+1/2 +1/6) = 1/6.$$
This fact combined with Equation~\eqref{eq:event} and the fact that for any $1 \leq i \leq K$, $\mathbb{P}_{\mathcal{G}^i}(\widehat S_{\tau} = \emptyset)  \geq \mathbb{P}_{\mathcal{G}^i}(\mathcal{E}_i)$ implies that for any $1 \leq i \leq K$
\begin{align}
\mathbb{P}_{\mathcal{G}^i}(\widehat{S}_{\tau} = \emptyset)  &\geq \frac{1}{6}\exp\Big(-6t_i\text{KL}_i -2 \ \ \sqrt[]{T\log(6TK)}\Big)\nonumber\\ 
&\geq  \frac{1}{6}\exp\Big(-60t_id_i^2 -2 \ \ \sqrt[]{T\log(6TK)}\Big),\label{eq:event2}
\end{align}
where we use Equation~\eqref{eq:gapKL} for the last step.

\paragraph{Step 4 : Conclusions.}

Since $\sum_{1 \leq k \leq K} d^{-2} = \log(K) H_f$, and since $\sum_{1 \leq k \leq K} t_k = T$, then there exists $1 \leq i \leq K$ such that
$$t_i \leq \frac{T}{\log(K) H_fd^2},$$
as the contraposition yields an immediate contradiction.
For this $i$, it holds by Equation~\eqref{eq:event2} that
\begin{align*}
\mathbb{P}_{\mathcal{G}^i}(\widehat S_{\tau} \neq \emptyset)  \geq\frac{1}{6}\exp\Big(-60\frac{T}{\log(K) H_f} - 2 \ \sqrt[]{T\log(6TK)}\Big).
\end{align*}
Note that $\; \forall\; k \in [K]$ we have $H_{tbp}(\mathcal{G}^0) = H_{tbp}(\mathcal{G}^k) = \sum_{k} d_k^{-2}$. Since the event $\{\widehat{S}_{\tau} = \emptyset\} \subseteq \mathcal{L}(T)$ for any problem instance $\mathcal{G}^{i}, i\in[K]$, and by assumption $T \geq 4(\log(K) H_f)^2\log(6TK)/(60)^2$ then we may subsume the second term and arrive at the desired result
\begin{align*}
\max_{i \in \{1, \ldots, K\}} 
        \mathbb{E}_{\mathcal{G}^i} (\mathcal{L}(T)) \ge \mathbb{P}_{\mathcal{G}^i}(\widehat S_{\tau} \neq \emptyset)  \geq\frac{1}{6}\exp\Big(\frac{-120T}{\log(K) H_f}\Big).
\end{align*}
\end{proof}

\subsection{Risky Class}

\subsection{Definition of Bandit Problems}
Let $\tau = 3/8$, $K, M \ge 2$. Let the attribute distributions of arm 1 be denoted as $\nu^{R}_{1j} = \mathcal{B}(\tau + d_{R})$ and $\nu^{R'}_{1j} = \mathcal{B}(\tau - d_{R})$ for all $j \in [M]$ and  $d_R = 1/8$ such that $\tau + d_R = 1/2$. Let the arm distribution of arm $1$ be $\nu^{R}_{1} = \bigotimes_{j \in [M]} \nu_{1j}^{R}$. Let $\beta \in (0,1)$ be some constant that parametrizes the difficulty of this class of problems. Define for $2 \leq i \leq K$, the distributions $\nu_{ij}^R := \mathcal{B}(\frac{1}{2} - d_i)$ and $\nu_{ij}^{R'} := \mathcal{B}(\frac{1}{2} + d_i)$ with 
\[
d_i = \frac{\beta i}{16K}\sqrt{\frac{K-1}{MK}}.
\] Let the arm distribution of arm $2 \leq i \leq K$ be defined as $\nu^{R}_{i} = \bigotimes_{j \in [M]} \nu_{ij}^{R}$. We now define $K+M$ bandit instances. Define the "base" bandit instance where arm $1=i^\star$ is feasible and the highest mean as 
\[\mathcal{G}_{R}^1 = \bigotimes_{i\in [K]}\nu_{i}^{R}. \] 
Note that in this case arm $1$ has the highest mean $\mu_1 = 1/2$ with all other arms $\mu_k = 1/2 - d_i < 1/2$. Further, all other arms are feasible, the since $MK \geq K-1 \implies d_K \leq 1/16 \implies \mu_K \geq 7/16 > \tau$.

The first $K-1$ problem instances are defined similarly as in \cite{carpentier2016tightlowerboundsfixed} where the $i^\star = j$ in problem instance $j$. This is because $\mu_j(\mathcal{G}^j_R) = 1/2 + d_i > \mu_1$. Formally, let for $2 \leq j \leq K$ and define 
\[
\mathcal{G}_{R}^{j} = \bigotimes_{i \in [K]/\{j\}} \nu_{i}^{R} \bigotimes_{j} \nu_{j}^{R'}.
\]
The next $M$ problem instances are infeasible in some attribute of arm $1$ and thus $i^\star=2$ and $1 \in \mathcal{F}^C$.
For $K+1 \leq j \leq M+K+1$ (i.e. $j-K \in [M]$) define 
\[
\mathcal{G}_{R}^{j} = \bigotimes_{i \in [K] - \{1\}} \nu_{i}^{R} \bigotimes_{i \in [M] - \{j-K\}} \nu_{1i}^{R} \bigotimes_{j-K} \nu_{1,j-K}^{R'}.
\]

Define the \emph{risky class} as $\mathcal{C_{R}}(\beta; K, M) =: \{\mathcal{G}_R^i: 1 \le i \le K+M+1\}$ where $\beta$ is a parameter that controls the gaps in each bandit problem $\mathcal{G}^i_R$ and $K,M$ affect the number of arms and attributes. This class of bandit instances are more difficult than simply identifying the arm with the highest mean since the learner must also learn to distinguish whether the best arm is feasible or not in addition to identifying the best arm.

\subsection{Risky Class Lower Bound Proof}

\begin{lemma}[Risky Class Lower Bound]
\label{lem:risky-class}
For any $ K,M \ge 2$, if $T \geq 4(\log(K)\max\{H_{tbp}(\mathcal{G}^i_R), H_2^R(\mathcal{G}^i_R)\})^2\log(6T(K+M+1))/60^2$, for any bandit algorithm we have the following lower bound
\begin{equation}
    \max_{\mathcal{G}^i_R \in \mathcal{C}_R(\beta; K,M)}\mathbb{P}_{\mathcal{G}^i_R }(e) \geq \frac{1}{6}\exp \left( \frac{-1200T}{\log(K) \max\{H_2^R(\mathcal{G}^i_R), H_{tbp}(\mathcal{G}^i_R)\}} \right).
\end{equation}
\end{lemma}

\begin{proof}

As in Lemma \ref{lem:tbpg}, we adapt the argumentation in \cite{carpentier2016tightlowerboundsfixed} to our novel bandit problem class construction. Further, we demonstrate that although the learner has a prior hint of when the best arm may be infeasible (since only arm $1$ is infeasible in our construction), we obtain bounds that are matched up to constant factors in the exponent by FCSR. This is possible since the risky class $\mathcal{C}_R$ is constructed carefully to also obtain the elusive $\log (K)$ factor that multiplies that hardness parameter. The missing $\log (K)$ factor in the lower bound was the reason there existed a gap between the upper bound and the lower bound in the vanilla best arm identification problem until it was resolved in \cite{carpentier2016tightlowerboundsfixed}.

\paragraph{Step 1: High probability event where empirical KL divergences concentrate}

Let $\nu_{i,j}^w \;\forall\; w\in[K+M+1]$ denote the distribution of attribute $(i,j)$ in problem instance $\mathcal{G}_R^w$. Define the empirical KL divergence of problem $w \in [K]$ wrt. problem $1$ as 
\[
\widehat{\text{KL}}_{w}(t) = 
\begin{cases}
\frac{1}{t}\sum_{s=1}^t \sum_{j=1}^M\log \frac{d\nu_{i,j}^1}{d\nu_{i,j}^w}(X(t)) &:w\in[K], i=w \\
\frac{1}{t}\sum_{s=1}^t \log \frac{d\nu_{1,j}^1}{d\nu_{1,j}^w}(X(t)) &:w\in[K+1:K+M+1], j=w-K
\end{cases}
\] 
where $X(t) \sim_{\text{i.i.d}} \nu^1_{i,j}$. Let $\text{KL}_w = \mathbb{E}_{\mathcal{G}_R^1}[\widehat{\text{KL}}_{w}(t)]$. The empirical KL divergence of problem $w: \widehat{\text{KL}}_{w}(t)$ gives the log-likelihood ratio of all samples observed under the base instance $\mathcal{G}_R^1$ to problem instance $\mathcal{G}_R^w$. The above definition only contains the non-zero terms that arise from the perturbed arms/attributes.

Define $d_w = d_i$ for $w, i \in [K]$ and $d_w = d_R = 1/8$ for $w \in [K+1:K+M+1]$. For $w \in [K]$, the perturbed arm distributions are Bernoulli distributions whose parameters are symmetric around $1/2$, thus $\text{KL}_w = \text{KL}(0.5-d_w, 0.5+d_w) =  \text{KL}(0.5+d_w, 0.5-d_w)$. For this case we have the following inequality for all $i,j$ (Equation (2) in \cite{carpentier2016tightlowerboundsfixed})
\[
\text{KL}_{w} \leq 10d_w^2.
\]
For $w \in [K+1:K+M+1]$ we have the following inequality 
\[
\text{KL}_w  = \text{KL}(0.5, 0.25) \leq 10d_R^2 = 10d_w^2.
\]
Define the event 
\begin{equation}
\xi' = \left\{ 
\forall\, w \in [K+M+1],\, 
\forall\, 1 \le t \le T,\,
\left| \widehat{\mathrm{KL}}_{w}(t) - \mathrm{KL}_{w} \right|
\le 2 \sqrt{ \frac{ \log(6T(K+M+1)) }{ t } }
\right\}.
\end{equation}

\paragraph{Claim.}\label{xi-r}
It holds that
$\mathbb P_{\mathcal{G}^1_R}(\xi') \geq 5/6.$

We verify the above claim. For $w \in [K]$, the mean parameter of arm $w$ is perturbed around $0.5$. More explicitly, $0.5 - d_w = \mu_{k,m}$ for $k=w$ and all $m$ under $\mathcal{G}_R^1$ and $0.5 + d_w = 1-\mu_{k,m}$ for $k=w$ and all $m$ under $\mathcal{G}_R^w$. $\mu_{w}\in [1/4,1/2)$ and thus,
$$|\log(\frac{d \nu_{k,m}^1}{d \nu_{k,m}^{w}}(X(s)))| = |\mathbf 1\{X(s) = 1\} \log(\frac{\mu_{k,m}}{1-\mu_{k,m}}) + \mathbf 1\{X(s) = 0\}\log(\frac{1-\mu_{k,m}}{\mu_{k,m}})| \leq \log(3),$$
where $X(s) \sim \nu_{k,m}^1$. Now for $w \in [K+M+1]$, only attribute $j = w-K$ of arm $1$ is perturbed and so we have
\[
|\log(\frac{d \nu_{1,j}^1}{d \nu_{1,j}^{w}}(X(s)))| = |\mathbf 1\{X(s) = 1\} \log(\frac{0.5}{0.25}) + \mathbf 1\{X(s) = 0\}\log(\frac{0.5}{0.75})| \leq \log(3).
\]
 Therefore, for any $w$, $\widehat{\text{KL}}_w(t)$ is a sum of i.i.d.~samples that are bounded by $\log(3)$, and whose mean is $\mathrm{KL}_{w}$. We can apply Hoeffding's inequality to this quantity for all $K, M, t$ and we have that with probability larger than $1-(6T(K+M+1))^{-1}$
$$|\widehat{\text{KL}}_{w}(t) - \text{KL}_{w}|\leq \sqrt{2}\log(3) \ \sqrt[]{\frac{\log(6T(K+M+1))}{t}}.$$
This assertion and an union bound over all $1 \leq w \leq K+M+1$ and $1 \leq t \leq T$ implies that $\mathbb P_{\mathcal{G}^1_R}(\xi) \geq 5/6$, as we have $\sqrt{2}\log(3) < 2$, and thus the claim is verified.

\paragraph{Step 2: A change of measure}
Let some arbitrary policy $\mathcal{A}lg$ return arm $I_T$ at the end of the budget $T$. Let $T_{k,m}$ denote the numbers of samples collected by $\mathcal{A}lg$ on each choice $(k,m)$ and $T_k = \sum_{m}T_{k,m}$ be the number of samples collected for arm $k$. 
Let the generalized index $w$ refer to arm $i \in [K]$ if $2 \leq w \leq K$, and refer to attribute $(1,w-K)$ if $K+1 \leq w \leq K+M+1$. Thus, $w$ indexes the set of perturbed arms/attributes. Define $T_w$ as 
\[
T_w = 
\begin{cases}
\sum_{j=1}^MT_{w,j} &: w \in [K], \\
T_{1,w-K} &:w\in[K:K+M+1].
\end{cases}
\]
It holds that $\sum_{w} T_{w} = \sum_{k=2}^K T_k + \sum_{m=1}^M T_{1,m}=\sum_{k,m}T_{k,m}=T$ by definition of the fixed budget setting. Let us write for any  $t_{k,m} = \mathbb E_{\mathcal{G}^1_R} [T_{k,m}]$, $t_{k} = \mathbb E_{\mathcal{G}^1_R} [T_{k}]$ and $t_{w} = \mathbb E_{\mathcal{G}^1_R} [T_{w}]$. Thus, it also holds that $\sum_{w} t_{w} = T$.

We recall the change of measure identity (see e.g. \cite{Audibert2010BestArmIdentification}) which states that for any measurable event $\mathcal{E}$ and for any $1 \leq w \leq K+M+1$ :
\begin{equation}
\label{cm-r}
\mathbb{P}_{\mathcal{G}^w_R}(\mathcal{E}) = \mathbb{E}_{\mathcal{G}^1_R} \Big[\mathbf{1}\{\mathcal{E}\}\exp \big(-T_w \widehat{\text{KL}}_{w}(T_{w})\big)\Big],
\end{equation}
as the product distributions $\mathcal G^w_R$ and $\mathcal G^1_R$ differ in all attributes that are perturbed in problem instance $w$. $T_w \widehat{\text{KL}}_{w}(T_{w})$ serves as the log-likelihood ratio. Let $1\leq w \leq K+M+1$. Consider now the event 
$$\mathcal{E}_w = \{I_T = 1\} \cap \{\xi'\} \cap \{T_w \leq 6 t_w\},$$
i.e.~the event where the algorithm reports arm $1$ as the best, where $\xi'$ holds, and where the number of times the set of perturbed attributes in $w$ was pulled is smaller than $6t_w$. We have by Equation~\eqref{cm-r}
\begin{align}
\mathbb{P}_{\mathcal{G}^w_R}(\mathcal{E}_w) & =\mathbb{E}_{\mathcal{G}^1_R}\Big[\mathbf{1}\{\mathcal{E}_w\}\exp\big(-T_w\widehat{\text{KL}}_{w}(T_w)\big)\Big]\nonumber\\
& \geq \mathbb{E}_{\mathcal{G}^1_R}\Big[\mathbf{1}\{\mathcal E_w\}\exp\Big(-T_w\text{KL}_w -2 \ \sqrt[]{T_i\log(6T(K+M+1)}\Big)\Big] \nonumber\\
& \geq \mathbb{E}_{\mathcal{G}^1_R}\Big[\mathbf{1}\{\mathcal E_w\}\exp\Big(-6t_w\text{KL}_w -2 \ \sqrt[]{T\log(6T(K+M+1))}\Big)\Big]\nonumber\\
& \geq \exp\Big(-6t_w\text{KL}_w -2 \ \sqrt[]{T\log(6T(K+M+1))}\Big) \mathbb{P}_{\mathcal{G}^1_R}(\mathcal E_w),\label{eq:event-r}
\end{align}
since on $\mathcal E_w$, we have that $\xi'$ holds and that $T_w \leq 6t_w$, and since $\mathbb E_{\mathcal{G}^1_R} \widehat{\text{KL}}_{w}(t) = \text{KL}_w$ for any $t \leq T$.

\paragraph{Step 3 : Lower bound on $\mathbb{P}_{\mathcal{G}^1_R}(\mathcal E_w)$ for any reasonable algorithm}
Assume that for the algorithm $\mathcal Alg$ that we consider
\begin{equation}\label{eq:bras0-r}
\mathbb P_{\mathcal{G}^1_R}(I_T \neq 1) \leq 1/2,
\end{equation}
i.e.~that the probability that $\mathcal Alg$ makes a mistake on problem $\mathcal{G}^1_R$ is less than $1/2$. Note that if $\mathcal Alg$ does not satisfy that, it performs badly on problem $\mathcal{G}^1_R$ and its probability of success is not larger than $1/2$ uniformly on the $K+M+1$ bandit problems we defined.

For any $1 \leq w \leq K+M+1$ it holds by Markov's inequality and $\mathbb E_{\mathcal{G}^1_R} [T_w] = t_w$ that,
\begin{align}\label{eq:ma-r}
\mathbb P_{\mathcal{G}^1_R} (T_k \geq 6 t_k) \leq \frac{\mathbb E_{\mathcal{G}^1_R} [T_w]}{6t_w} = 1/6.
\end{align}
So by combining Equations~\eqref{eq:bras0-r},~\eqref{eq:ma-r} and Claim~\ref{xi-r}, it holds by an union bound that for any $2 \leq w \leq K+M+1$
$$\mathbb P_{\mathcal{G}^1_R}(\mathcal E_i) \geq 1 - (1/6+1/2 +1/6) = 1/6.$$
This fact combined with Equation~\eqref{eq:event-r} and the fact that for any $1 \leq w \leq K+M+1$, $\mathbb{P}(e) \geq \mathbb{P}_{\mathcal{G}^w_R}(I_T = 1)  \geq \mathbb{P}_{\mathcal{G}^w_R}(\mathcal{E}_i)$ implies that for any $1 \leq w \leq K+M+1$
\begin{align}
\mathbb{P}_{\mathcal{G}^w_R}(e)  &\geq \frac{1}{6}\exp\Big(-6t_w\text{KL}_w -2 \sqrt{T\log(6T(K+M+1))}\Big)\nonumber\\ 
&\geq  \frac{1}{6}\exp\Big(-60t_wd_w^2 -2 \sqrt{T\log(6T(K+M+1))}\Big).\label{eq:event2-r}
\end{align}

\paragraph{Step 4 : Conclusions.}


We first upper bound the thresholding hardness in our construction. Denote $H_{tbp}^{(i)} := H_{tbp}(\mathcal{G}^i_R).$ Since the only arm in our constructions that is infeasible is arm $1$, we have
\begin{equation}
\label{eq:lbc1}
H_{tbp}^{(i)} = K\max_{i\in{\mathcal{F}^C}}\sum_{j}\bar\Delta_{ij}^{-2} \leq K \cdot Md_{R}^{-2} = 64MK,
\end{equation}
where in the definition of $H_{tbp}$ arms are indexed by their true mean. The above inequality is obtained by picking the maximum index $K$ and $\sum_{j}\bar\Delta_{ij}^{-2} = Md_R^{-2}$ for the potentially only infeasible arm. Define the classical hardness $H^{(i)} = H(\mathcal{G}^i_R) := \sum_{i \neq 1}\Delta_i^{-2}$. 

We may lower bound the classical hardness parameter of any bandit problem $\mathcal{G}^i_R$ by considering the easiest instance $\min_{2 \le i \le K+M+1 } H^{(i)}$. This is the instance $K$ where arm $K$ is flipped. Notice here $\Delta_{i} \geq (2d_K)^{-2}$ for any other arm $i \neq K$. Thus, for any $i \in [K+M+1]$
\begin{equation}
\label{eq:lbc2}
(K-1)(2d_K)^{-2} \leq \min_{2 \le i \le K+M+1 } H^{(i)} \leq H^{(i)}.
\end{equation}

From $\beta \in (0,1)$ and the definition of $d_K$  
\begin{align*}
d_K = \frac{\beta}{16}\sqrt{\frac{K-1}{MK}} \implies (K-1)(2d_K)^{-2} = (K-1)\frac{256MK}{4 \beta^2 (K-1)} \geq  64MK.
\end{align*}
Combining the above result with \eqref{eq:lbc1} and \eqref{eq:lbc2} we have for any $i$
\begin{equation}
H^{(i)} \geq H_{tbp}^{(i)}.
\end{equation}

Now let us define $H_{R}(i) := H_{tbp}^{(i)} + H^{(i)}$, a combined hardness parameter that upper bounds the desired $H_{FC}$ and let $h^* = \sum_{2 \le i \le K+M+1}\frac{1}{d_i^2H_R(i)}$. There exists $2 \leq i \leq K+M+1$ such that \[t_i \leq \frac{T}{h^{*}d^2H_R(i)}.\] 
Plugging this into \eqref{eq:event2-r} we have,  
\begin{equation}
\label{eq:gen-lb-af}
    \mathbb{P}_{\mathcal{G}^i_R}(e) \geq
    \mathbb{P}_{\mathcal{G}^i_R}(I_T = 1) \geq 
    \exp(-\frac{60T}{h^{*}H_R(i)} - 2\sqrt{T\log(6T(K+M+1))}).
\end{equation}

We now provide a lower bound on $h^\star$. For any $1\leq i\leq K$, we have 
$$d_i^2H(i) = d_i^2 \sum_{k\neq i} \frac{1}{(d_i+d_k)^2} \leq d_i^2\Big( \frac{i}{d_i^2} + \sum_{k> i} \frac{1}{d_k^2}  \Big) \leq i + i^2\sum_{K \geq k \geq i} \frac{1}{k^2} \leq i +i^2 (\frac{1}{i} - \frac{1}{K})\leq  2i.$$
This implies that
\begin{equation}
    d_i^2H_R(i) = d_i^2H_{tbp}(i) + d_i^2H(i) \leq 2d_i^2H(i) \le 4i.
\end{equation}
Thus, using an integral to lower bound the Riemann sum, we obtain the following lower bound on $h^\star$:
\begin{align*}
    h^{*} & \geq \sum_{2 \le i \le K+M+1}\frac{1}{d_i^2H_R(i)}
    \ge \sum_{2 \le i \le K}\frac{1}{d_i^2H_R(i)} \\
    &\ge \frac{1}{4}\sum_{i=2}^{K}\frac{1}{i} \geq \frac{1}{4} (\log(K+1)-\log(2)) \ge  \frac{1}{10}\log(K),
\end{align*}
where the last inequality can be verified empirically. For all problem instances the set of risky arms $\mathcal{R} = \emptyset$, and so $H_2(\mathcal{G}^i_R) = \max_{i\neq i^\star}i\Delta_{i}^2 \leq H^{(i)}$. (This is a well known result in fixed budget best arm identification literature; see \cite{carpentier2016tightlowerboundsfixed}). Thus, $H_{R}(i) \geq \max\{H_{tbp}(\mathcal{G}^i_R), H_2(\mathcal{G}^i_R)\}$. Hence,  plugging our lower bound on $h^\star$ in \eqref{eq:gen-lb-af}, we show that there exists some bandit instance $\mathcal{G}^i_R$ in $\mathcal{C}_{\mathcal{R}}$ such that 
\begin{equation}
    \label{eq:lowerrisk}
    \mathbb{P}_{\mathcal{G}^i_R}(e) \geq \exp\left(\frac{-600T}{\log(K)\max\{H_{tbp}(\mathcal{G}^i_R), H_2^R(\mathcal{G}^i_R)\}} - 2\sqrt{T\log(6T(K+M+1))}\right).
\end{equation}
Choosing $T \geq 4(\max\{H_{tbp}(\mathcal{G}^i_R), H_2^R(\mathcal{G}^i_R)\})^2\log(6T(K+M+1))/600^2$, the second term can be lower bounded by the first and we arrive at the desired result.

\textbf{Note:} From \eqref{eq:lbc2} we have $H^{(i)} \geq \max\{H_{tbp}(\mathcal{G}^i_R), H_2^R(\mathcal{G}^i_R)\} $. It can be verified that the base instance has the highest classical difficulty $H^{(i)}$ and thus for all $i$ we have
\begin{equation}
\label{eq:lbb-r}
 \max\{H_{tbp}(\mathcal{G}^i_R), H_2^R(\mathcal{G}^i_R)\} \leq H^{(1)} \leq \sum_{i=2}^{K} (\frac{\beta i}{16 K})^2\frac{K-1}{MK} \leq \frac{\beta^2K}{256M}.
\end{equation}

\end{proof}


\subsection{Complete Lower Bound}

\begin{theorem}[Lower Bound]
 Let $\mathcal{C}_{FC}(a;K)$ denotes the set of bandit instances with $K=M \geq 2$ and whose difficulty $H_{FC}$ is upper bounded by some constant $a$. If
 \begin{equation}
 \label{eq:t-bound}
 T \geq \frac{4}{60^2}(a\log(K))^2\log(6T(K+M+1))
 \end{equation}
 there exists a bandit instance $\mathcal{G} \in \mathcal{C}_{FC}(a)$ such that the probability any arbitrary learner incorrectly reports the best arm is at least$$
\mathbb{P}_{\mathcal{G} \in \mathcal{C}_{FC}(a)}\geq \frac{1}{6}\exp\left( \frac{-1200T}{\log(K)H_{FC}(\mathcal{G})} \right).
$$
\end{theorem}

\begin{proof}

Choose $K=M \geq 2$ and $a = \max\{\beta^2/256, \frac{M}{\log(K)}d^{-2}\}$.  We have from \eqref{eq:lbb-r} in Lemma \ref{lem:risky-class} that \[\max \{H_2^R, H_{tbp}\} \leq \beta^2/256.\] Further, from  \eqref{eq:lbb-f} in Lemma \ref{lem:tbpg} we have $H_f \leq \frac{Md^{-2}}{\log(K)}$. Thus, $\mathcal{C}_R(\beta; K, M) \subset \mathcal{C}_{FC}(a)$ and $\mathcal{C}_F(\beta; K) \subset \mathcal{C}_{FC}(a)$. Hence, the difficulty of any problem in the class $\mathcal{C}_{FC}$, $H_{FC} \leq a$ and the problem class $\mathcal{C}_F$ is rich enough to apply Lemma \ref{lem:risky-class} and Lemma \ref{lem:tbpg}. Further, if $T$ satisfies the bound given in \eqref{eq:t-bound}, it automatically satisfies the required bounds to use the above lemmas, since $\log(6TK) < \log(6T(K+M+1))$. 

The lower bound on the error probability of any arbitrary algorithm $Alg$ over $\mathcal{C}_{FC}$, is at least the maximum of the worst case error probability over the individual sub-classes $\mathcal{C}_R, \mathcal{C}_F$, i.e.
\begin{equation}
\label{eq:lb-max}
 \max_{\mathcal{G} \in \mathcal{C}_{FC}} \mathbb{P}_{\mathcal{G}}(e) \geq \max\left\{
 \max_{\mathcal{G} \in \mathcal{C}_F }\mathbb{P}_{\mathcal{G}} (e), 
 \max_{\mathcal{G} \in \mathcal{C}_R}\mathbb{P}_{\mathcal{G} }(e)
 \right\}.
\end{equation}

We use Lemma \ref{lem:tbpg} to  to establish the lower bound for the first term in \eqref{eq:lb-max} and Lemma \ref{lem:risky-class} for a lower bound on the second term. Choosing the constant that minimizes the overall expression over the numerator, taking the max inside the denominator of the exponent and from the definition of $H_{FC}(\mathcal{G}) = \max\{H_{tbp}(\mathcal{G}), H_2^R(\mathcal{G}), H_f(\mathcal{G} \}$, we obtain 
$$
\mathbb{P}_{\mathcal{G} \in \mathcal{C}_{FC}(a;K)}(e) \geq \frac{1}{6}\exp\left( \frac{-1200T}{\log(K)H_{FC}(\mathcal{G})} \right).
$$

\end{proof}

\section{PROOF OF UPPER BOUND ON FCSR ERROR}
\label{app:fcsr}
\begin{lemma}
\label{lem:ville}
\textbf{(Ville's inequality.)} 
Let \((M_t,\mathcal F_t)_{t\ge0}\) be a nonnegative supermartingale with \(M_0\le1\). Then for any \(a>0\),
\[
\mathbb{P}\big(\sup_{t\ge0} M_t \ge a\big) \le \frac{1}{a}.
\]
\end{lemma}


Let $s_r(i)$ denote the score of arm $i$ after all samples allocated in round $r \in \{1, 2, \dots, K-1\}$ by FCSR. Define  $\mathcal{F}_r := \{i:s_r(i) > \tau\}$ as the set of all arms that are deemed feasible at the end of round $r$ and $\mathcal{F}_r^c := [K]/\mathcal{F}_r$, the set of all arms that are deemed infeasible. 

\begin{lemma}
\label{app:suf}
Assuming the best arm $i^\star$ exists, the probability of $i^\star$ being deemed infeasible at the end of round $r \in \{ 1, 2, \dots, K-1 \}$ is  upper bounded by
\begin{equation}
\mathbb{P}(i^\star\in\mathcal F_r^c)\le \exp\!\Big(-\frac{fT}{16R^2\log K\,H_f}\Big),
\end{equation}
for all $T \geq \max \{\frac{4K^2M}{f}\}$, where $f \in (0, 1)$ is a static hyperparameter. Recall that $H_{f} = \frac{K}{\log K} \max \bar{\Delta}_{i^\star j}^{-2}$. 
\end{lemma}

\begin{proof}

Recall that the FCSR algorithm (Algorithm \ref{alg:fcsr}) allocates for each round $r \in [K-1]$, an equal number of samples to all surviving arms in the following fashion: all $M$ attributes are first sampled uniformly (Algorithm \ref{alg:us}), then according to the APT routine (Algorithm \ref{alg:apt}) and finally according to the SUF routine (Algorithm \ref{alg:suf}). 

\paragraph{Definitions.}
Fix arm $i^\star$, let $A_j^{(s)}$ denote the number of times attribute $j \in [M]$ of $i^\star$ is sampled by the Uniform and APT routines, i.e. the number of samples $j$ receives in round $s$ before the invocation of SUF.

Let $\mathcal{H}_{s}$ be the sigma algebra generated by all observations uptil the end of round $s-1$. Define $T^{(s)}_j$ to be the number of times attribute $j$ has been sampled by FCSR before round $s$ and let $S^{(s)}_j$ be the corresponding sum of the $T^{(s)}_j$ rewards observed. Define the \emph{offset}
\[
C^{(s)}_j := S^{(s)}_j - T^{(s)}_j\,\tau.
\]
Notice that $C^{(s)}_j > 0$ if the empirical mean $S^{(s)}_j/T^{(s)}_j > \tau$ and $j$ was deemed feasible in round $s-1$. Thus if $C^{(s)}_j \leq 0$, the empirical mean of $j$ was less than or equal to $\tau$ at the end of $s-1$ (where $s -1> 1$). The latter case only arises if the feasibility budget $P_1$ was exhausted in round $s-1$, otherwise by definition SUF would have continued sampling the attribute until the empirical mean crossed $\tau$.

Let $X_j(t)$ denote the $t$-th sample of attribute $j$ (assumed to be \(R\)-sub-Gaussian with mean \(\mu_j\)). we may define the stream of samples SUF shall see when called on attribute $j$ in round $s$ as
\[
Y^{(s)}_j(u) := X_{j}\big(T^{(s)}_j+u\big),\qquad u\ge1.
\]

Let $Z_j^{(s), \text{act}}$ denote the number of samples allocated to attribute $j$ in round $s$. We define the quantity $Z_j^{(s)}$ that closely approximates $Z_j^{(s), \text{act}}$ as
\[
Z_j^{(s)}:=
\begin{cases}
\inf\Big\{t\ge0:\ \dfrac{S^{(s)}_j+\sum_{u=1}^t Y^{(s)}_j(u)}{T^{(s)}_j+t} >\tau\Big\}, 
&\text{if } C^{(s)}_j > 0,\\[6pt]
0, &\text{if } C^{(s)}_j \leq 0,
\end{cases}
\]
with the convention \(\inf\emptyset=+\infty\). Note that $Z_j = Z_j^{(s), \text{act}}$, except in the case $P_1$ is exhausted midway through sampling attribute $j$ (in which case $Z_j > Z_j^{(s), \text{act}}$ ).  This is because if $C^{(s)}_j \leq 0$, $j$ must have been deemed infeasible in round $s-1$ implying that the feasibility budget $P_1$ must have been exhausted. Thus, $Z_j^{(s)} =0 = Z_j^{(s), \text{act}}$. If $C^{(s)}_j > 0$ and $P_1$ is not exhausted midway, SUF continues sampling till the empirical mean crosses $\tau$ exactly as in the definition of $Z_j$. Thus, we have pathwise
\[
Z_j^{(s)} \geq Z_j^{(s), \text{act}}.
\]
Define the totals
\[
S^{\mathrm{act}}:=\sum_{s=1}^r\sum_{j=1}^M Z^{(s),\mathrm{act}}_j,\qquad
S:=\sum_{s=1}^r\sum_{j=1}^M Z^{(s)}_j.
\]
\(S^{\mathrm{act}}\le P_1\) and \(S\ge S^{\mathrm{act}}\) holds always. If the algorithm declares arm \(i^\star\) infeasible by the end of round \(r\), then necessarily, the feasibility budget $P_1$ must have been exhausted during the executed SUF calls up to round \(r\); and since $\{S \geq P_\} \implies \{S^{\text{act}} = P_1\}$, we have
\[
\{i^\star \in \mathcal{F}_r^c\}\subseteq\{ S \ge P_1\}.
\tag{I}
\]
Therefore,
\[
\mathbb{P}(i^\star \in \mathcal{F}_r^c) \leq \mathbb{P}(S \ge P_1),
\]
and it suffices to bound $\mathbb{P}(S \ge P_1)$. From inclusion (I) and Chernoff
\begin{equation}
\mathbb{P}( S\ge P_1)
\le e^{-\theta P_1}\,\mathbb E[e^{\theta S}].
\end{equation}
We now bound $E[e^{\theta S}]$ by first bounding $E[e^{\theta Z^{(s)}_j}]$.

\paragraph{Per-attribute mgf.}

Fix any round $s \in \{1, \dots, r\}$ and attribute $j$ and condition on \(H_s\). We first bound ${P}(Z_j^{(s)} \geq t)$ in the case (i) \(C_j^{(s)} \leq 0\) and (ii) \(C_\ell >  0\).

(i) If \(C_j^{(s)} \leq 0\) then \(Z^{(s)}_j=0\) and
\begin{equation}
\mathbb{P}(Z_j^{(s)} \geq t) =0. 
\end{equation}

(ii) If $C_j^{(s)}$, define the centred future sum $\widetilde S_j^{(s)}(t)$ and the exponential process $M_\ell(t;\lambda)$ as
\[
\widetilde S_j^{(s)}(t):=\sum_{u=1}^t\big(Y^{(s)}_j(u)-\mu_{i^\star j}\big), 
\qquad M_\ell(t;\lambda):=\exp\!\Big(-\lambda\widetilde S_\ell(t)-\tfrac{\lambda^2R^2}{2}t\Big).
\]
For any \(\lambda>0\), the process $M_\ell(t;\lambda)$ 
is a nonnegative supermartingale (w.r.t.\ \(\sigma(\mathcal H_{s},Y^{(s)}_j(1),\dots,Y^{(s)}_j(t))\)) with \(M_\ell(0;\lambda)=1\). On the event \(\{Z^{(s)}_j\ge t\}\)  ($t \geq 1$), we may obtain the following inequality using the definition of $Z_j^{(s)}$ and subtracting $(T^{(s)}_j+t)\mu_{i^\star j}$ from both sides:
\[
\dfrac{S^{(s)}_j+\sum_{u=1}^{t-1} Y^{(s)}_j(u)}{T^{(s)}_j+t-1} \leq \tau \implies
\widetilde{S}_{j}^{(s)}(t-1) \leq - (t-1)\bar{\Delta}_{i^\star j} - C_j^{(s)}.
\]
Thus, on \(\{Z_j^{(s)}\ge t\}\),
\[
M(t-1;\lambda)\ge \exp\!\Big(\lambda\big((t-1)\bar\Delta_{i^\star j}+C_j^{(s)}\big)-\tfrac{\lambda^2R^2}{2}(t-1)\Big).
\]

Since $\sup_{t\ge0} M(t;\lambda) \ge M(t-1;\lambda)$ we apply Ville's inequality (Lemma \ref{lem:ville}) conditional on \(\mathcal H_{s}\), the history of all samples before round $s$, with $a = \exp\!\Big(\lambda\big((t-1)\bar\Delta_{i^\star j}+C_j^{(s)}\big)-\tfrac{\lambda^2R^2}{2}(t-1)\Big)$ to obtain
\[
\mathbb{P}\big(Z_j^{(s)}\ge t\mid\mathcal H_{s}\big)
\le \mathbb{P}\!\big(\sup_{t\ge0}M_t \ge a\big)
\le
\exp\!\Big(-\lambda\big((t-1)\bar\Delta_{i^\star j}+C_j^{(s)}\big)+\tfrac{\lambda^2R^2}{2}(t-1)\Big).
\]
The term \(-\lambda C_j^{(s)}\) is non-positive (since $C_j^{(s)} > 0$) and can be dropped.  Optimizing in \(\lambda\) (take \(\lambda=\bar\Delta_{i^\star j}/R^2\)) yields, for \(t\ge1\),
\[
\mathbb{P}\big(Z_j^{(s)}\ge t\mid\mathcal H_{s}\big)
\le \exp\!\Big(-\frac{\bar\Delta_{i^\star j}^2}{2R^2}\,(t-1)\Big).
\]
 Define \(\alpha_j:=\bar\Delta_{i^\star j}^2/(2R^2)\) and \(\alpha:=\min_j\alpha_j\). Hence, uniformly across both cases (i) and (ii), for all $Z_j^{(s)}$ we have
\[
\mathbb{P}\big(Z_j^{(s)}\ge t\mid\mathcal H_{s}\big)\le e^{-\alpha (t-1)}\qquad(t\ge1).
\]

For \(0<\theta<\alpha\),
\[
\begin{aligned}
\mathbb E\big[e^{\theta Z_j^{(s)}}\mid\mathcal H_{\ell-1}\big]
&=1+\sum_{t\ge1}(e^{\theta t}-e^{\theta(t-1)}) \mathbb{P}(  Z_j^{(s)} \ge t\mid\mathcal H_{s})\\
&\le 1+\sum_{t\ge1}(e^{\theta t}-e^{\theta(t-1)})e^{-\alpha (t-1)}\\
& \le 1+\frac{e^{\theta}-1}{1-e^{-(\alpha-\theta)}}
=\frac{e^{\theta}-e^{-(\alpha-\theta)}}{1-e^{-(\alpha-\theta)}}
=:G(\theta).
\end{aligned}
\]

\paragraph{Overall mgf.}
From the independence of attributes we have,
\begin{align*}
\mathbb{E}\big[e^{\theta S}\big]
&= \mathbb E\big[e^{\theta \sum_{s=1}^r\sum_{j=1}^M Z^{(s)}_j}\big]
= \mathbb \prod_{j=1}^M E\Big[  e^{\theta \sum_{s=1}^{r} Z_j^{(s)}} \Big].
\end{align*}

Using the iterated rule of expectations and the conditional independence of future samples of attribute $j$ in round $s$ wrt previous rounds, we obtain the following 
\[
E\Big[  e^{\theta \sum_{s=1}^{r} Z_j^{(s)}} \Big] \leq 
\mathbb E\Big[ \mathbb E\big[ e^{\theta Z_j^{(r)}} e^{\theta \sum_{s=1}^{r-1} Z_j^{(r)}}\,\big|\,\mathcal H_{r}\big]\Big]
\leq 
E\Big[ e^{\theta \sum_{s=1}^{r-1} Z_j^{(r)}} \Big]\mathbb E \big[ e^{\theta Z_j^{(r)}} \,\big|\,\mathcal H_{r}\big]
\leq 
E\Big[ e^{\theta \sum_{s=1}^{r-1} Z_j^{(r)}} \Big] G(\theta)
\]
where the last inequality follows from the per-call bound $\mathbb{E}\big[ e^{\theta Z_j^{(s)}} \,\big|\,\mathcal H_{s}\big] \leq G(\theta)$. Iterating this conditioning argument for \(s=r-1,r-2,\dots,1\) yields
\[
E\Big[  e^{\theta \sum_{s=1}^{r} Z_j^{(s)}}] \leq G(\theta)^r.
\]
Finally,
\[
\mathbb E\big[e^{\theta S}\big] = 
\prod_{j=1}^M E\Big[  e^{\theta \sum_{s=1}^{r} Z_j^{(s)}} \Big]
\le G(\theta)^{Mr}.
\]

Thus, from inclusion (I) and Chernoff,
\[
\mathbb{P}(i^\star\in\mathcal F_r^c)
\le \exp\!\big(-\theta B_1 + Mr\log G(\theta)\big),
\qquad 0<\theta<\alpha.
\]
Choose \(\theta=\alpha/2\). If $B_1 \ge \frac{4}{\alpha}\,Mr\log G(\tfrac{\alpha}{2})$, then \(Mr\log G(\alpha/2)\le \tfrac{\alpha}{4}B_1\), and therefore
\[
-\theta B_1 + Mr\log G(\theta)
= -\tfrac{\alpha}{2}B_1 + Mr\log G(\tfrac{\alpha}{2})
\le -\tfrac{\alpha}{4}B_1.
\]
Hence
\[
\mathbb{P}(i^\star\in\mathcal F_r^c)\le \exp\!\Big(-\frac{\alpha}{4}B_1\Big).
\]
Recall \(B_1=\lfloor fT/K\rfloor\). For any \(x\ge 1\) we have \(\lfloor x\rfloor \ge x/2\). Hence, if $\frac{fT}{K} \ge 2$, then $B_1=\Big\lfloor \frac{fT}{K}\Big\rfloor \ge \frac{1}{2}\,\frac{fT}{K}$. Under this condition, the bound is 
\begin{equation}
\label{eq:suf-final2}
\mathbb{P}(i^\star\in\mathcal F_r^c)\le \exp\!\Big(-\frac{\alpha}{8}\,\frac{fT}{K}\Big).
\end{equation}

Equivalently for $B_1 \ge \frac{4}{\alpha}\,Mr\log G(\tfrac{\alpha}{2})$, using $fT/K \geq 2$ and $B_1 \geq fT/2K$ it suffices to require
\[
T \ge \frac{8K}{f\alpha}\,Mr\log G(\tfrac{\alpha}{2}).
\]

Since $H_f = \frac{K}{\log K} \max_{j} \bar{\Delta}^{-2}_{i^\star j}$, we obtain \(\alpha=\frac{K}{2R^2\log K\,H_f}\). Substituting this into equation \ref{eq:suf-final2} we obtain the desired bound
\[
\mathbb{P}(i^\star\in\mathcal F_r^c)\le \exp\!\Big(-\frac{fT}{16R^2\log K\,H_f}\Big),
\]
for all \(T \ge Mr\log K\frac{ 16R^2H_f }{f}\log G(\frac{K}{4R^2\log K\,H_f})\). It can be verified that the crude upper bound $G(\theta) < \theta$ holds. Thus we require $T \geq \max \{\frac{4KMr}{f}\}$ which is true if  $T \geq \max \{\frac{4K^2M}{f}\}$. 

\end{proof}

\begin{lemma}[APT with $P$ random interjections]
\label{lem:apt-p}
Let $K$ arms be $R$-sub-Gaussian with means $\mu_i$. Fix threshold $\tau$ and set $\varepsilon=0$. Define
\[
\Delta_i:=|\mu_i-\tau|,\qquad H:=\sum_{i=1}^K \Delta_i^{-2}.
\]
Let the total draw budget be $T_{\mathrm{tot}}=T+P$, and assume among these $T_{\mathrm{tot}}$ draws there exists a subsequence of exactly $T$ time-steps at which the algorithm pulled arms according to APT (call these the \emph{APT-sampled} steps). Running APT on those $T$ APT-sampled steps yields, for the class $\mathcal B_{H,R}$ of problems with complexity at most $H$, the guarantee: if
\[
T \ge 256\,H\,R^2\log\big((\log T +1)K\big),
\]
then
\[
\sup_{B\in\mathcal B_{H,R}}\mathbb{E}\big[\mathcal L(T)\big] \le \exp\!\Big(-\frac{T}{128\,R^2\,H}\Big).
\]
\end{lemma}

\begin{proof}
The proof in \cite{locatelli2016optimalalgorithmthresholdingbandit} proceeds by defining a favourable event $\xi$ where empirical means concentrate and show that APT never makes a mistake on $\xi$. The $P$ additional samples in our setting do not interfere with the probability of the concentration $\xi$. Thus, the proof is identical to the one in Appendix A.2 of \cite{locatelli2016optimalalgorithmthresholdingbandit} with the modification that all empirical means, counts and indices below are computed only over the $T$ APT-sampled pulls out of the total $T+P$ pulls. Formally, let $t_1 < t_2, \dots, <t_T$ be the indices of the samples pulled in accordance with the APT rule. Index these samples as $1, 2, \dots, T$, and the rest of the non APT samples as $T+1, \dots, T+P$ again in increasing order of the true time step.  

Let $\mathbb{A}$ be a set of $K$ $R$-sub-Gaussian arms and let \(\delta = (4\sqrt{2})^{-1}\). Define the concentration event \(\xi\) as follows:

\[\xi = \left\{ \forall i \in \mathbb{A}, \forall s \in \{1, ..., T\} : \left| \frac{1}{s} \sum_{t=1}^{s} X_{i,t} - \mu_i \right| \leq \sqrt{\frac{T\delta^2}{Hs}} \right\}.
\]

The rest of the proof follows directly since the probability of the concentration event $\mathbb{P}(\xi)$ can only increase with additional samples even if they do not sample the $K$ arms according to the APT rule.

\end{proof}

\begin{theorem}[Performance of FCSR]
Let $K, M \geq 1$. Given parameters $f, g \in (0,1)$ and $\tau$. Let $c = 1/32R^2$ be a constant and $R$ be the sub-Gaussianity parameter. The probability of error of FCSR($\tau, f, g$) satisfies
\begin{equation}
\mathbb{P}(e) \leq 3K^2\exp\left( -\frac{cT}{\log(K)H_{FC}(\mathcal{B})} \right)
\end{equation}
for all 
\begin{align*}
T \geq \max \left\{ 256H_{tbp}R^2\log((\log(T) + 1)M), \frac{4K^2M}{f}, \frac{K\max\{{H_{tbp}}^{-1}, (H_{2}^{R})^{-1}\}}{\log(K)} \right\}.
\end{align*}
\end{theorem}

\begin{proof}
Our proof follows a similar style as in \cite{Audibert2010BestArmIdentification}. Define $n_i = \lceil \frac{n - K}{\bar{\log(K)}(K+1-i)}\rceil$. Index all arms in order of their true arm mean. During round $i$, at least one of the worst $i$ arms survives. If the best arm $i^\star$ is eliminated in round $1 \leq i \leq K-1-|\mathcal{R}|$, at least one sub-optimal arm with index $|\mathcal{R}|+2\leq j \leq K$ survives and thus must beat $i^\star$, or the best arm $i^\star$ is deemed infeasible. Denote the former event by $s_i$ and the latter by $f_i = \{i^\star \in \mathcal{F}^C_i\}$.

Now, if the best arm $i^\star$ is eliminated in round $K-|\mathcal{R}| \leq i \leq K-1$, at least one risky arm with index $1\leq j \leq |\mathcal{R}|$ must be deemed feasible in order to beat the best arm $i^\star$. This event is a subset of the event that there exists \emph{some} infeasible arm $r \in \mathcal{F}^C$ that is deemed feasible at the end of round $i$. Denote this event by $r_i$. Thus, we may upper bound the probability the best arm in this case is eliminated by the probability that it is either deemed infeasible, or in the event of $r_i$. 

From a union bound of the above two cases and all rounds $i \in [K-1]$, we upper bound the probability that $i^\star$ is not flagged the best arm as
\begin{equation}
\label{eq:ub-1}
\sum_{i=1}^{K-1}\mathbb{P}(e_i) \leq \sum_{i=1}^{K-1}\mathbb{P}(f_i) + \sum_{i=1}^{K-1 - |\mathcal{R}| }\mathbb{P}(s_i)+ \sum_{K-|\mathcal{R}|}^{K-1}\mathbb{P}(r_i).
\end{equation}

We bound $\mathbb{P}(f_i) = \mathbb{P}(i^\star \in \mathcal{F}_r^C)$ using Lemma \ref{lem:ville} since by assumption we have $T \geq \max \{\frac{4KMr}{f}\}$. Thus for any $i \in [K-1]$,
\begin{equation}
\label{eq:fib}
\sum_{i=1}^{K-1}\mathbb{P}(f_i)\le K\exp\!\Big(-\frac{fT}{16R^2\log K\,H_f}\Big).
\end{equation}

To bound $\mathbb{P}(r_i)$, notice that for any infeasible arm to be deemed feasible, i.e. $r\in\mathcal{F}_i$, it must be deemed feasible after being subjected to $\lfloor g(1-f)n_i \rfloor$ APT samples. By assumption we have $T \ge 256\,H_{tbp}\,R^2\log\big((\log T +1)M\big)$ and therefore, from Lemma \ref{lem:apt-p} and a union bound we have
\begin{equation}
\label{eq:rib}
\sum_{K-|\mathcal{R}|}^{K-1}\mathbb{P}(r_i) \leq \sum_{K-|\mathcal{R}|}^{K-1}\sum_{r\in\mathcal{F}^C} \mathbb{P}(r\in\mathcal{F}_i) \leq K^2\exp\left( \frac{-g(1-f)(T-K) - 1}{\overline{\log}(K)\max_{r\in\mathcal{F}^C}KH_{tbp}(r)} \right),
\end{equation}
where the last inequality follows from taking the max inside the exponent and $(K+1-i) \leq K$ and lower bounding the floor. Finally, the bound on $\mathbb{P}(s_i)$ is obtained from the fact that for rounds $1 \leq i \leq K-1-|\mathcal{R}|$, the probability that some arm with index $|\mathcal{R}|+2\leq j \leq K$ beats $i^\star$ is
\begin{equation}
\label{eq:sib}
\begin{aligned}
\sum_{K-|\mathcal{R}|}^{K-1}\mathbb{P}(s_i) \leq \sum_{K-|\mathcal{R}|}^{K-1}\sum_{k=K+1-i}^K{\mathbb{P}(\hat{\mu}_{i^\star} \leq \hat{\mu}_{k}  )} 
& \leq K^2\exp \left( -\lfloor(1-f)(1-g)n_k\rfloor \Delta_{(K+1-i)}^2\right) \\
& \leq K^2\exp \left( \frac{-(1-f)(1-g)(T-K) - 1}{\overline{\log}(K)H_2^R}\right).
\end{aligned}
\end{equation}

The above inequality follows from the fact that $\lfloor (1-g)(1-f)n_i\rfloor$ samples have been allotted according to the APT rule up till round $i$, and from the lower bound on the floor. We may bound $\overline{\log}(K) \leq 2\log(K)$. Since $f, g \in (0,1)$ and from our assumption on $T$ we have 
\[
T \geq \frac{K}{\log(K)}\max\{{H_{tbp}}^{-1}, (H_{2}^R)^{-1}\} \implies T \geq 2\max\left\{\frac{g(1-f)K}{2\log(K)H_{tbp}}, \frac{(1-g)(1-f)K}{2\log(K)H_{2}}\right\}.
\]
Thus, we may replace $(T-K)$ by $T/2$ in \eqref{eq:sib} and \eqref{eq:rib} and neglect the $-1$ in the numerator. Finally, we obtain the desired result by upper bounding \eqref{eq:ub-1} by taking a max and using \eqref{eq:fib}, \eqref{eq:rib}, \eqref{eq:sib} in addition with the fact that $f,g \in (0,1)$. Thus,
\begin{equation}
\mathbb{P}(e) \leq \sum_{i=1}^{K-1}\mathbb{P}(e_i) \leq 3K^2  \exp\left( \frac{-T }{32R^2\log(K)\max\{H_f, H_{tbp}, H_2^R\}} \right).  
\end{equation}

\end{proof}





\end{document}